\documentclass[10pt,twocolumn,letterpaper]{article}
\usepackage[camera ready]{cvpr} 

\usepackage{graphicx}
\usepackage{amsmath}
\usepackage{amssymb}
\usepackage{booktabs}

\usepackage{multirow}
\usepackage{booktabs}       
\usepackage{amsfonts}       
\usepackage{nicefrac}       
\usepackage{microtype}      
\usepackage{color}
\usepackage{wrapfig,lipsum,booktabs}
\usepackage[utf8]{inputenc}
\usepackage{mathtools}  
\usepackage{amsmath}
\usepackage{anyfontsize}
\usepackage{amssymb}
\usepackage{tabulary}
\usepackage{booktabs}
\usepackage{algorithm}
\newcommand{\argmin}{\mathop{\arg \min}}
\newcommand{\argmax}{\mathop{\arg \max}}
\usepackage{algorithm,algpseudocode}
\usepackage{amsmath}
\usepackage{amsfonts}
\usepackage{amssymb}
\usepackage{amsthm}
\makeatletter
\newtheorem{theorem}{Theorem}[section]
\newtheorem{lemma}[theorem]{Lemma}
\newtheorem{assumption}{Assumption}

\newcommand{\ours}{\textit{SGPT}}
\usepackage{subcaption}

\usepackage{xspace}
\newcommand{\select}{\texttt{Select}\xspace}

\newcommand{\cut}[1]{\textcolor{red}{}}

\newcommand{\Yc}{\mathcal{Y}}
\newcommand{\Dc}{\mathcal{D}}
\newcommand{\Xc}{\mathcal{X}}

\newcommand{\Cc}{\mathcal{C}}

\newcommand{\Hc}{\mathcal{H}}

\newcommand{\beq}{\begin{equation}}
\newcommand{\eeq}{\end{equation}}
\newcommand{\bea}{\begin{align}}
\newcommand{\eea}{\end{align}}

\newcommand{\nn}{\notag}

\DeclarePairedDelimiterX{\inp}[2]{\langle}{\rangle}{#1, #2}

\usepackage[dvipsnames]{xcolor}

\definecolor{cvprblue}{rgb}{0.21,0.49,0.74}
\usepackage[pagebackref,breaklinks,colorlinks,citecolor=cvprblue]{hyperref}

\def\paperID{*****} 
\def\confName{CVPR}
\def\confYear{2024}

\title{Unlocking the Potential of Prompt-Tuning in Bridging Generalized and Personalized Federated Learning}

\author{Wenlong Deng, Christos Thrampoulidis$^{*}$, Xiaoxiao Li\thanks{Equal Corresponding Author.}  \\
Department of Electrical and Computer Engineering,\\ The University of British Columbia, Vancouver, BC, Canada \\
\texttt{\{dwenlong,cthrampo,xiaoxiao.li\}@ece.ubc.ca}
}

\begin{document}
\maketitle

\begin{abstract}
Vision Transformers (ViT) and Visual Prompt Tuning (VPT) achieve state-of-the-art performance with improved efficiency in various computer vision tasks. This suggests a promising paradigm shift of adapting pre-trained ViT models to Federated Learning (FL) settings. However, the challenge of data heterogeneity among FL clients presents a significant hurdle in effectively deploying ViT models. Existing Generalized FL (GFL) and Personalized FL (PFL) methods have limitations in balancing performance across both global and local data distributions. In this paper, we present a novel algorithm, \ours{}, that integrates GFL and PFL approaches by employing a unique combination of both shared and group-specific prompts. This design enables \ours{} to capture both common and group-specific features. A key feature of \ours{} is its prompt selection module, which facilitates the training of a single global model capable of automatically adapting to diverse local client data distributions without the need for local fine-tuning. To effectively train the prompts, we utilize block coordinate descent (BCD), learning from common feature information (shared prompts), and then more specialized knowledge (group prompts) iteratively. Theoretically, we justify that learning the proposed prompts can reduce the gap between global and local performance. Empirically, we conduct experiments on both label and feature heterogeneity settings in comparison with state-of-the-art baselines, along with extensive ablation studies, to substantiate the superior performance of \ours{}.

\end{abstract}    
\section{Introduction}
\label{sec:intro}
FL is a framework that allows machine learning models to be learned from multiple clients without sharing their data~\cite{mcmahan2017communication}. In the landscape of computer vision, the integration of ViT~\cite{dosovitskiy2020image} with FL emerges as a pivotal research domain as it promises a significant improvement in image recognition tasks. Firstly, ViT's attention mechanism has demonstrated exceptional ability in deriving robust and discriminative representations~\cite{dosovitskiy2020image,bommasani2021opportunities,vaswani2017attention,bhojanapalli2021understanding} in terms of scalability and adaptability to a variety data scenarios; a feature crucial in FL environments. Secondly, ViTs demonstrate a remarkable ability to generalize from limited data by leveraging powerful publicly available pre-trained ViT models~\cite{dosovitskiy2020image,he2022masked,jia2022visual,khan2022transformers} as initialization, making them inherently suitable for FL’s decentralized nature. 
\begin{figure}[t!]
    \centering
    \includegraphics[width=\linewidth]{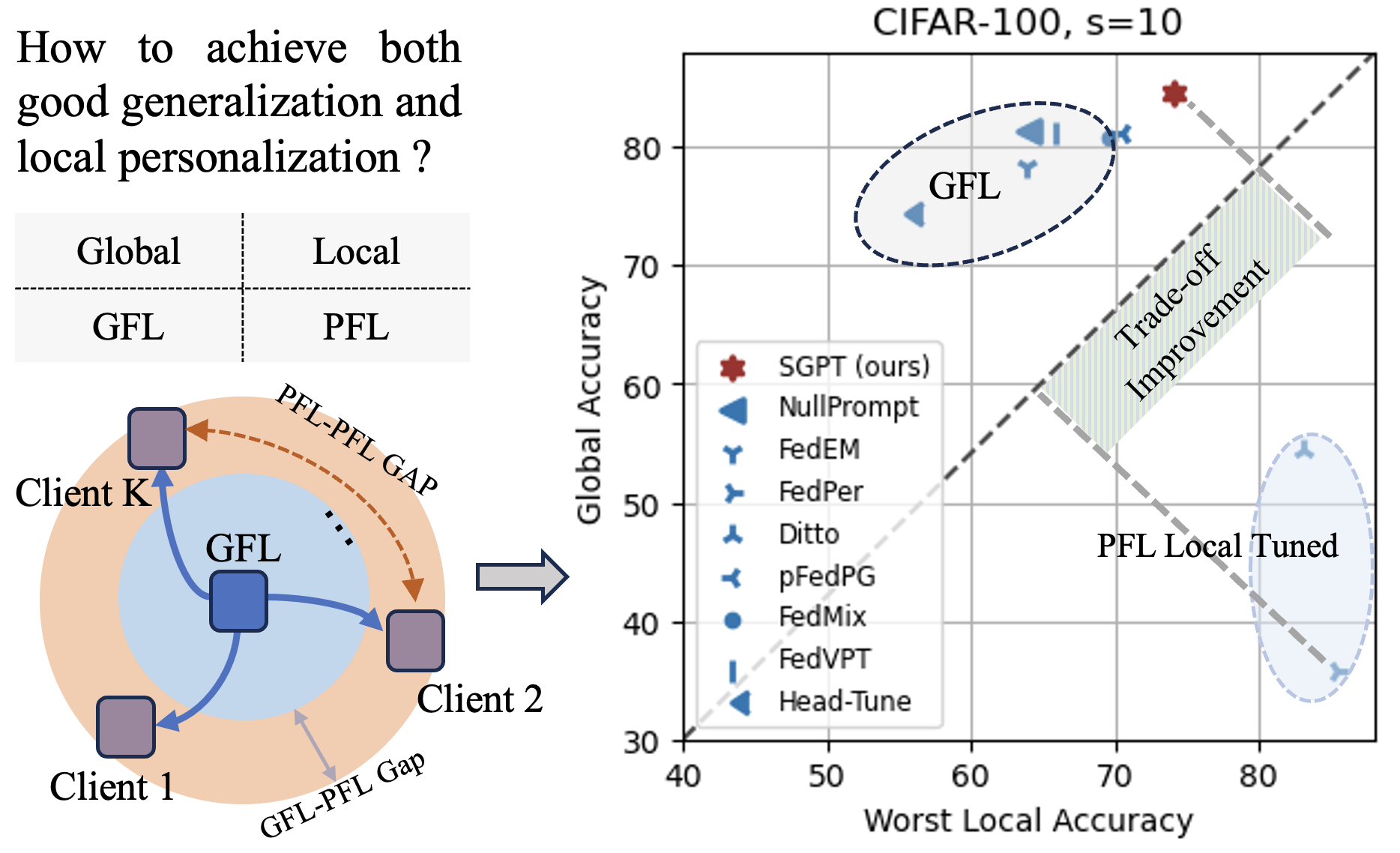}
    \caption{Global accuracy and worst local accuracy on CIFAR-100 with $s=10$ ($s$ is the number of classes per client). Points located in the top-right corner correspond to great performance on both the global data and local clients’ data distributions. PFL models perform well on local data, however, lack the ability to predict out-of-client data. Global models have a better generalization but cannot well adapt to each local data distribution. Our proposed \ours{} ($\star$) achieves the best trade-off.
 }\label{fig:100pareto}
 \vspace{-3mm}
\end{figure}
While ViTs are often seen as computationally demanding, recent advancements have significantly enhanced their efficiency casting them suitable for FL\footnote{A detailed related work review on parameter efficient tuning for transformer is provided in the appendix.}. These improvements include 1) updating part of model parameters~\cite{sun2022exploring} or 2) optimizing additional parameters with a frozen model~\cite{yi2023fedlora,li2022fedtp}, during local training of FL followed by federated averaging~\cite{mcmahan2017communication}. In this paper, we focus on the latter approach by employing Visual Prompt Tuning (VPT)~\cite{jia2022visual} given its efficiency and effectiveness in vision tasks~\cite{jia2022visual,sun2022exploring,yang2023efficient}.

Although prompt tuning techniques allow efficient FL, applying them to certain FL scenarios still remains an open research challenge~\cite{yang2023efficient}, particularly when data across clients exhibits heavy heterogeneity in terms of domain discrepancy~\cite{li2021fedbn} or imbalanced class distribution~\cite{li2021model}. To navigate \textit{data heterogeneity}, two primary strategies are used: Generalized FL (GFL) and Personalized FL (PFL). On the one hand, GFL approaches focus on learning a single global FL model that achieves high generalization, with methods like FedAvg variants~\cite{karimireddy2020scaffold,li2021fedbn,li2020federated}. On the other hand, PFL tailors models to individual clients and clustered client groups. For instance, some PFL methods~\cite{fallah2020personalized,wu2022motley} involve local data fine-tuning to customize a global model for personalized models. Others~\cite{ghosh2020efficient, mansour2020three, caldarola2021cluster} take into account client similarity and integrate clients with similar data distributions into clusters. In the regime of FL with ViT, FedPR~\cite{feng2023learning}, a recent prompt-tuning-based GFL method, learns client prompts and aggregates them into global prompts. Recently, FedPG~\cite{yang2023efficient}, another PFL method, uses a Hyper-Network~\cite{shamsian2021personalized} to generate client-specific prompts. 
Both GFL and PFL methods have their own limitations: (1) GFL methods are insufficient when dealing with significant data heterogeneity~\cite{li2020federated} with one global model; (2) PFL customizes client models, which can lead to overfitting on local data~\cite{wu2022motley}. Overall, this limits their ability to generalize to other distributions and may not allow them to adapt to out-of-federation clients.\footnote{A detailed related work review on FL is provided in the appendix.}

To overcome the limitations of GFL and PFL methods under data heterogeneity, it is essential to leverage their respective strengths through a combination of the two. We show that this is possible by appropriately leveraging prompt-tuning. 
Concretely, we achieve this by developing a new FL algorithm \underline{S}hared 
and \underline{G}roup \underline{P}rompt \underline{T}uning (dubbed \ours{}). Our algorithm focuses on learning a shared global model during training, allowing it to acquire global information, while also enabling local adaptation with prompt selection. This approach leads to high accuracy in generalizing to global distribution as well as efficient alignment to various local client data distributions (see Fig.~\ref{fig:100pareto}). To elaborate, firstly, \ours{} globally learns \textbf{shared} and \textbf{group} prompts, which facilitate the learning of both universal and group-specific knowledge. Secondly, the prompt selection module (see Fig.\ref{fig:main} (b)) effectively finds data groups and assigns group prompts to each input, thus \textit{automatically aligning the global model with local distributions} (see Fig.\ref{fig:main} (c)), without needing local fine-tuning. Thirdly, we use block-coordinate-descent (BCD) for effective parameter training, starting with learning common features (shared prompts) before optimizing group prompts for specialized knowledge iteratively. 
Finally, we present a theoretical error bound and identify two factors \textit{generalization} and \textit{distribution discrepancy} that affect \textit{the gap between the global and local performance}. Our \ours{} effectively considered these two terms. 
In summary, our contributions are as follows.
\begin{itemize}
\item We introduce \ours{}, a novel approach that employs shared prompts to capture common information and utilizes group prompts to effectively align the global model with local distributions via a group selection module without local fine-tuning.  
\item We introduce a BCD optimization routine that iterates between learning the easy and common knowledge (through shared prompts) and then the more complex and specific knowledge (through group prompts). This way, we tackle optimization-specific challenges in the algorithm implementation.
\item We theoretically bound the gap between global and local performance and identify two key factors: generalization and distribution discrepancy. \ours{} can affect both these two factors, thus tightening the bound.
\item  We empirically test our algorithm on a wide range of datasets and types of heterogeneities. The results are compelling, demonstrating that \ours{} consistently surpasses all baseline models.
\end{itemize}

\begin{figure*}[t] 
\centering
\includegraphics[width=\linewidth]{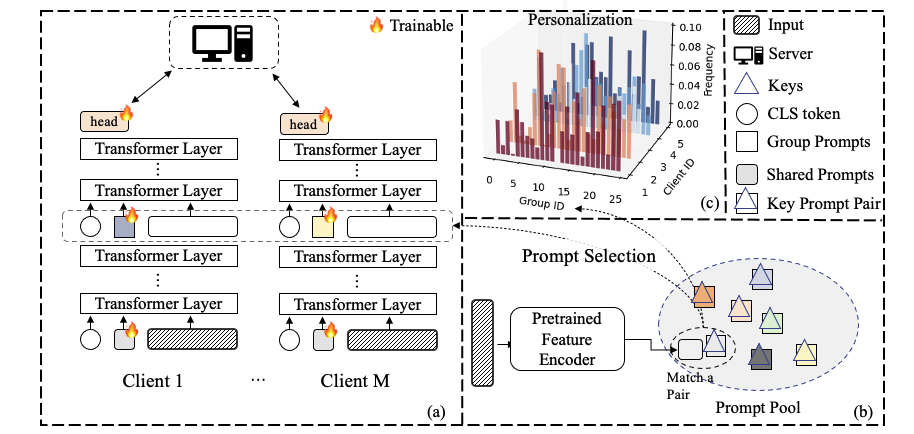} 
\caption{Pipleline of our method. 
(a) Provides an overview of the federated group-aware prompt-tuning \ours{} procedure. Each model comprises shared prompts and group prompts, facilitating the acquisition of both common and group-specific knowledge. The shared prompts and classification head are globally trained, while the group prompt is inserted into intermediate layers trained within its respective data group and shared globally. (b) Depicts the prompt selection module. Here, each input undergoes processing by a pre-trained ViT model encoder. Similarities between keys and last layer CLS token features are calculated, and the prompt corresponding to the most similar key is selected for training, enabling group-aware training at the sample level. (c) Given that data distributions vary across clients, the frequencies of selected group prompts differ, ensuring our model aligns with various local data distributions.}
\label{fig:main}
\vspace{-4mm}
\end{figure*}

\section{Problem Setting and Preliminary}\label{sec:setting}
\subsection{Problem Setting}\label{sec:assump}
In this paper, we examine a scenario involving 
$M$ clients and a central server. The clients' data distributions are heterogeneous, characterized by either domain discrepancies or imbalanced class distributions. We denote the distribution for client $i$ as $\mathcal{D}_i$ with $i \in [M]~\footnote{[m]=\{0, 1, \dots, m\}}$. Each client $i$ contains $N_i$ data samples $\{(x_j^i,y_j^i)\}_{j=1}^{N_i}$. Further, let the parameter of a pre-trained ViT model be \textcolor{blue}{$\theta$} that are \textcolor{blue}{frozen} during training. Denote \textcolor{red}{trainable} prompts as \textcolor{red}{$P$}, and classifier weights as \textcolor{red}{$W_{C}$}. We introduce the objective function of prompt-tuning the task model in FL:
\begin{align}\label{eq:obj1}
    \argmin_{\textcolor{red}{P},\textcolor{red}{W_{C}}} \sum^M_{i=1}\frac{N_i}{N}\sum^{N_i}_{j=1}l(\textcolor{blue}{\theta},\textcolor{red}{P},\textcolor{red}{W_{C}};x^i_j,y^i_j).
\end{align}
where $l: \mathcal{X} \times \mathcal{Y} \rightarrow \mathbb{R}^{+}$ is the cross-entropy loss and $N$ is the total number of data on all clients. In this way, we can leverage the representation power of ViT while enabling efficient tuning by only learning the prompts and classifier.


\subsection{Visual Prompt Tuning (VPT)}\label{sec:vpt_f}
Prompt Tuning is an efficient alternative to full fine-tuning for large-scale Transformer models. VPT~\cite{jia2022visual} introduces only a small amount (less than 1\% of model parameters) of trainable parameters in the input space while keeping the backbone model frozen. Depending on the number of Transformer layers involved, ~\cite{jia2022visual} proposed VPT and VPT-D for efficient fine-tuning. To be specific, for VPT and VPT-D, prompts are inserted into the first Transformer layer and all layers respectively. Take VPT as an example, the prompt token is a learnable $d$-dimensional vector. The learnable prompt $P$ is trained as follows:
\begin{align*}
\left[cls_1, Z_1, E_1\right]  & =\textcolor{blue}{L_1}\left(\left[\textcolor{blue}{cls_0}, \textcolor{red}{P}, E_0\right]\right) \\
\left[cls_u, Z_u, E_u\right]  & =\textcolor{blue}{L_u}\left(\left[cls_{u-1}, Z_{u-1}, E_{u-1}\right]\right) ~ u=2,\ldots, U \\
y & =\textcolor{red}{\operatorname{Head}}\left(x_U\right),
\end{align*}
where $U$ is the number of layers; $Z_u$ represents the prompt features computed by the $u$-th Transformer layer $L_{u}$; $E_u$ is a collection of image patch embeddings as the inputs to the $(u+1)$-th Transformer layer $L_{u+1}$; and $cls_u \in \mathbb{R}^d $ denotes the classification embeddings at $L_{u+1}$'s input space.

\section{Method}\label{sec:method}

\subsection{Architecture}\label{sec:arc}
In this section, we follow VPT~\cite{jia2022visual} to utilize a pre-trained ViT and adapt it to new tasks by prompt-tuning. We further advanced the architecture so that it can align to different data distributions automatically. Specifically, our approach involves learning shared prompts for common features while employing specialized group prompts through a prompt selection module to align with various local client data distributions effectively. Following VPT (shown in Section~\ref{sec:vpt_f}), we initialize with a pre-trained ViT model comprising $U$ layers but learn shared and group prompts (see Fig.~\ref{fig:main} (a)). First, we define the selection function $\select: \mathcal{X} \rightarrow [G]$ (detailed in Section.~\ref{sec:orthokey}) that efficiently groups data into $G$ groups. Then we introduce the shared and group prompts:

\noindent \textbf{Shared Prompts:}  The shared prompts $P_S$ are designed to capture common representations. Recent studies~\cite{ostapenko2022continual,hayes2020remind} have shown low-level representations can be shared across groups, and distilling commonly used information into shared prompts can enhance the model's generalization. Motivated by the observation that features from different classes, processed by early layers of a pre-trained ViT, are uniformly distributed on the manifold (shown in Fig.~\ref{fig:early} in Appendix), indicating shared information across classes. Therefore, we attach shared prompts $P_S$ to the embedding features of the early layers, \ie the first layer:
\begin{align}
\left[cls_1,Z^S_1,E_1\right]  & =\textcolor{blue}{L_1}\left(\left[\textcolor{blue}{cls_0}, \textcolor{red}{P_S}, E_0\right]\right).
\end{align}
\noindent \textbf{Group Prompts:} 
 The group prompts set $P_G = \{p_1, \dots, p_G \}$ containing $G$ group prompt $ p_g, g \in [G]$, which is designed to extract specialized information. In contrast to the early layers, the diverse and specialized features have shown to be preserved in higher layers~\cite{raghu2021vision}. Thus, we use the $\select$ function to assign group membership $g = \select(x)$ to a sample $x$ (detailed in Eq.~\ref{eq:select}) and insert corresponding group prompts $p_g \in P_G$ to higher layers (\ie the $u$-th layer) to extract task-specific features \cite{raghu2021vision}: 
\begin{align}
\left[cls_u,z_u^g,Z_u^S,E_u\right]  & \! = \! \textcolor{blue}{L_u}\left(\left[\textcolor{blue}{cls_{u-1}},\textcolor{red}{p_{g}},Z_{u-1}^S, E_{u-1}\right]\right),
\end{align}
where $Z^{S}_u$ represents the shared prompt features, $z_u^g$ represents the group prompt feature computed by the $u$-th Transformer layer $L_{u}$. At last, we rewrite the objective function Eq.~\eqref{eq:obj1} into:
\begin{align}\label{eq:pravg}
    \argmin_{P_G, P_S, W_C} \! \sum^M_{i=1} \! \frac{N_i}{N} \sum^{N_i}_{j=1}l(\textcolor{blue}{\theta}, \textcolor{red}{p_{g} ,P_S,W_{C}};x^i_j,y^i_j),
\end{align}
where $p_{g} \in P_G$ are the selected group prompts that aligned with the predicted group membership of input $x$ from the function $\select$. 

\subsection{Learning Prompt Selection Function }\label{sec:orthokey}
In this section, we introduce the prompt selection module $\select$ that learns data-specific keys for picking group prompts.
These keys are trained to capture the similarity of feature representations within their respective groups. 
Here, we propose a simple yet effective similarity-based clustering approach: we learn a key $k_g$ for each group $g \in [G]$ as a centroid~\cite{wang2022learning} and cluster data to its nearest centroid. Specifically, we first process the input sample $x$ using the pre-trained model $h_{\theta}$ to obtain its generalized feature representation~\cite{tripuraneni2020theory}. Subsequently, the keys $K = 
\{k_1,...,k_{G}\}$ then cluster feature representations into groups based on cosine similarity:
\begin{align} \label{eq:select}
\select(x) = \argmax_{g\in [G]} cos(h_{\theta}(x),k_{g}).
\end{align}
 However, training keys $K$ in the FL setting faces several challenges: 1) collapse, leading to data clustering in few groups~\cite{mahon2023hard} and 2) instability due to heterogeneous client data causing inconsistent clustering (as shown in Fig~\ref{fig:stab} (a)). 
 
To avoid collapse, we calibrate the \select function in Eq.~\ref{eq:select} by weighting with the accumulated selection probability $q_g$  during training. Specifically, let $v^t_g = \sum_{t'=1}^{t}\sum_{i=1}^{M}v^{t'}_{g,i} $ as the total number of times group $g$ is selected up to communication round $t$ across all clients $i$. The selection probability at communication round $t$ is then calculated as $q^t_g = \frac{v^t_g}{\sum_{g \in [G]} v^t_g}$, we drop the communication round $t$ to lighten notation in the later part. Then, we calculate the following loss function:
\begin{align} \label{eq:loss_k}
&L_{key} = -cos(h_{\theta}(x),k_{g})\\
& \text{where}~ g \in \argmax_{g \in [G]} \left({\rm cos}\big(h_{\theta}(x),k_g \big) -1 \right)  \cdot q_g, \nonumber
\end{align}

\begin{figure}[t!]
    \centering
\includegraphics[width=0.95\linewidth, height=3.5cm]{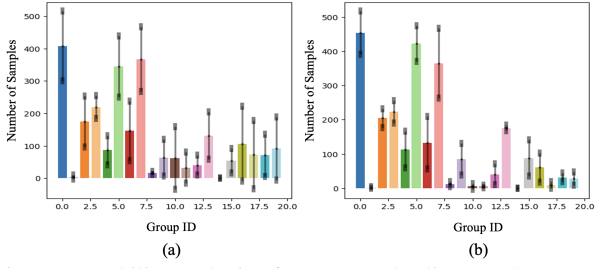}
     \vspace{-3mm}
    \caption{Stability analysis of one example client on CIFAR-100 dataset with $s=10$. We plot the mean and standard deviation of the prompt selection number overall communication rounds. (a) Without stability regularization, the variance is larger and is unstable. (b) With our proposed momentum updating, the variance is reduced and is more stable. }\label{fig:stab}
    \vspace{-6mm}
\end{figure}
To enforce the stability of clustering, we perform momentum parameter aggregation~\cite{he2020momentum,tarvainen2017mean} on the server side for both keys and group prompts to ensure selection consistency and knowledge consistency, respectively. Denoting the aggregated parameters for a group $g$'s key and prompts at round $t$ as $k_g^t$ and $p_g^t$, and the momentum parameters as $\hat{k}_g^t$ and $\hat{p}_g^t$ (see Algorithm~\ref{alg:training} in appendix), the parameters are updated as follows:
\begin{align}\label{eq:sta_key}
   \hat{k}_g^t &= \alpha_k \hat{k}_g^{t-1} + (1-\alpha_k) k_g^t, \\
     \hat{p}_g^t &= \alpha_g \hat{p}_g^{t-1} + (1-\alpha_g) p_g^t , \quad g\in[G], \notag
\end{align}
where $\alpha_k$ and $\alpha_g$ are the momentum rates ($0.5$ in our case).

\subsection{Block Coordinate Descent for Optimization}\label{sec:bcd}
Because we use \select function, a non-continuous function, for selecting group prompts, this renders the objective function in Eq.~\eqref{eq:pravg} to be non-smooth, further introducing optimization challenges. To effectively navigate the challenges, we employ a Block Coordinate Descent (BCD) method for the local training of Eq.~\eqref{eq:pravg}. The BCD optimization involves dividing the parameters into sub-groups and optimizing them in an iterative manner. In our case, this means optimizing the shared prompts and group prompts separately. Furthermore, the order of updating each parameter sub-group is critical for achieving good performance~\cite{shi2016primer}. Inspired by cognitive development theories on human learning progression~\cite{elman1993learning}, our method first learns the easy, common feature information (shared prompts) and then optimize the group prompts to extract more specialized knowledge.\footnote{To demonstrate the effectiveness of this ordered approach, we present an ablation study on the update sequences in Table~\ref{tab:ablation}.} In this way, we formulate the local objective functions for client $i$ as:
\begin{align}
    & \argmin_{P_G,W_C} \mathbb{E}_{(x,y) \! \sim \! \mathcal{D}_i} [ l(p_{\select(x)},P^{\star}_S,W_C;x,y)] ,\label{eq:stage1}\\
    & st. P^{\star}_S,W_C \in  \argmin_{P_S,W_C} \mathbb{E}_{(x,y) \! \sim \! \mathcal{D}_i} [ l(P_S,W_C;x,y)]. \label{eq:stage2}
\end{align}
We first find the optimal share prompts $P_S$ independently to learn common information. Then, the group prompts $P_G$ will residually learn group-specific knowledge upon the optimal shared prompts.
\begin{algorithm}[t!]
\textbf{Input:} Weights $W = \{W_C,P_G,P_S\}$, pre-trained ViT $h_{\theta}$, prompt selection module $\select$ with learnable keys $K =\{k_{g}\}_{g=1}^{G}$, training data   $(x,y) \sim \mathcal{D}$,  learning rate $\eta$, local training steps $E$ in one communication round. 
 \begin{algorithmic}[1]%
 \State \textbf{Block I:} Learn Shared Prompts Only:
\For{$e = 1\rightarrow E$}
 \State Optimize the Eq.~\eqref{eq:stage1}
 \State $ \{P^{\star}_S, W_C\}  \leftarrow \{P_S, W_C\} - \eta\cdot\nabla l(P_S,W_C;x,y) $ 
\EndFor
  \State \textbf{Block II:} Learn Group Prompts with frozen $\textcolor{blue}{P^{\star}_S}$ :
\For{$e = 1\rightarrow E$}
\State \Call{$p_g, k_g \gets$PromptSelection}{(x)}
\State $W \leftarrow W - \eta\cdot\nabla l(p_g,\textcolor{blue}{P^{\star}_S},W_C;x,y) $ 
 \algorithmiccomment{ Eq.~\eqref{eq:stage2}} 
\State $K \leftarrow K - \eta\cdot\nabla l_{key}(k_g;x,y) $ \algorithmiccomment{ Eq.~\eqref{eq:loss_k}} 
\EndFor
\Procedure{PromptSelection}{$\select,x$}
\State $p_g \leftarrow P_{\select(x)}$ \algorithmiccomment{Select group prompt} 
\State $k_g \leftarrow K_{\select(x)}$ \algorithmiccomment{Select group key}
\EndProcedure
 \end{algorithmic}
 \caption{Block Coordinate Descent on Local Client}
 \label{alg:BCD}
\end{algorithm} 

\noindent\textbf{Block I: Learning Shared Prompts.} The shared prompts are learned independently (lines 1-5 in Algorithm~\ref{alg:BCD}) to capture common information with the following loss function:
\begin{align}
    L_{share} = l(P_S,W_C;x,y)
\end{align}
where $l$ is the cross-entropy loss.

\noindent \textbf{Block II: Learning Group Prompts and Keys}: Having obtained the shared prompts, the group prompts are learned (lines 6-11 in Algorithm~\ref{alg:BCD}) to extract group-specific knowledge. Initially, the shared prompts $P_S$ are inserted but frozen. Then, we employ \select function to determine $x$'s group $g = \select(x)$ (detailed in Eq.~\eqref{eq:select}) and insert corresponding group prompts $p_g$ to extract group-specific features \cite{raghu2021vision}. The feature representation derived from these prompts is combined with the cls token via average pooling for the final classification. Consequently, the group prompts effectively learn group-specific knowledge. Additionally, The keys $K$ in \select function are learned simultaneously, and the total loss function is:
\begin{align}
    L_{total} = L_{key} + l(P_G,W_C;x,y),
\end{align}
where $l$ is cross-entropy and $L_{key}$ is the loss for learning \select function. We perform Algorithm~\ref{alg:BCD} to optimize the parameters over the global communication rounds.
\subsection{Efficient Inference}
In this section, we explain the inference procedure of \ours{}. Given a sample $x$, we first use the \select function (see Eq.~\eqref{eq:select}) to determine its group membership $g = \select(x)$. Then, the shared prompt $P_S$ and corresponding group prompt $p_g$ are inserted into the model to achieve sample-level adaptation for inference. When performing tests on new clients, the frequencies of selected group prompts can be automatically adjusted by \select function (shown in Fig.~\ref{fig:main} (c)), ensuring our model aligns with their local data distributions. We provide more training and inference details of our proposed \ours{} in the Appendix.
%

\section{Theoretical Analysis}
In this section, we provide analytical justification for narrowing the empirical risk of the global model found by empirical loss minimization and the population risk of the optimal model of a client. Following the heterogeneity setting in~\cite{marfoq2021federated}, we assume each local client $i$'s data distribution $\mathcal{D}_i$ is a mixture of $G$ underlying distributions (groups).
\begin{assumption}
On a client $i$, there exist $G$ underlying (independent) distributions $\mathcal{D}^i_g$, $g\in [G]$, such that for $i \in [M]$, $\mathcal{D}_i$ is mixture of the distributions $\mathcal{D}^i_g$ with mixing probability vector $\pi_i = [\pi_1^i,\ldots, \pi_G^i]$:
\begin{align} \mathcal{D}_i=\sum_{g\in[G]}\pi_g^i\,\mathcal{D}_g^i, \quad \sum_g \pi_g^i =1 \quad \text{and} \quad \pi_g^i \geq 0,
\end{align}
where $\pi_g^i = \frac{N_g^i}{N_i}$ is the probability of a data sample on client $i$ belong to group $g$ and $N^i_g$ is a fixed number of samples from $\mathcal{D}_g^i$ for all $g\in [G], i \in [M]$. 
\end{assumption}
Based on this, we also introduce the probability distribution $\mathcal{C}_g$ of data belonging to group $g$ as follows.
For a group $g$, its global data distribution $\mathcal{C}_g$ is a mixture of the distributions $\mathcal{D}^i_g, i \in [M]$ with mixing probability vector $\pi_g = [\pi_g^1,\ldots, \pi_g^M]$:
\begin{align}
      \mathcal{C}_g = \sum_{i\in[M]} \pi_g^i \,\mathcal{D}^i_g.
\end{align}
Following \cite{hugeneralization}, we refer to the distribution $\mathcal{C}_g$ as the ``participated clients' data distribution'' for the $g$-th group. 

Finally, for $g\in[G]$ and hypothesis $h_g \in \mathcal{H}$, we use $h_{\select} = \{h_1,...,h_g,...,h_G\}_{\select(x)}$ to denote the group-aware hypothesis determined by function \select when datapoint $x$ is given as input.
Let $\widehat{h}_g = \arg\min_{h\in \mathcal{H}}\mathcal{L}_{\widehat{\mathcal{C}}_g}(h)$ denote the empirical model for data group $g$ and $\hat{h}_{\select} = \{\hat{h}_1,...,\hat{h}_g,...,\hat{h}_G\}_{\select(x)}$ denote the corresponding global model.  
\begin{theorem}[Gap between the global and local performance] \label{th:ourtheory2} Assume the loss function $\ell$ is bounded in $[0,1]$ and the function \select is a data grouping method. Let the VC-dimension of hypothesis class $\mathcal{H}$ be $d$. Then,
with a probability of at least $1-\delta$ over the training set,
\begin{align}\label{eq:the}
\mathcal{L}_{\widehat{\mathcal{D}}_i}(\widehat{h}_{\select})  &  -\min _{h \in H} \mathcal{L}_{\mathcal{D}_i}(h)
   \leq \\ & \sqrt{\frac{\log \frac{1}{\delta}}{N_i}}  + \underbrace{2\sum_{i=1}^G \frac{N^i_g}{N_i} \sqrt{\frac{2d}{N_g}\left(1+\log(\frac{N_g}{d})\right )} }_{\text{Generalization}} \notag \\
    +  &\underbrace{\sum_{g=1}^{G} \frac{N^i_g}{N_i} \left(\operatorname{disc}(\mathcal{D}_g^i, \mathcal{C}_{g}) + \operatorname{disc}(\widehat{\mathcal{D}}_g^i, \widehat{\mathcal{C}}_{g}) \right)}_{\text{Distribution~Discrepancy}},\nonumber
\end{align}  
where  $\operatorname{disc}_{\mathcal{H}}\left(\mathcal{D}_1, \mathcal{D}_2\right)=\max _{h \in \mathcal{H}}\left|\mathcal{L}_{\mathcal{D}_1}(h)-\mathcal{L}_{\mathcal{D}_2}(h)\right|$ and $N_g$ is the tocal number of data in group $g$ from all the clients.
 \end{theorem}
The detailed proof is provided in the Appendix. The left-hand side of Eq.~\eqref{eq:the} represents the gap between the minimum empirical risk\footnote{We also give the gap on population distribution in the Appendix.} of the global model found by empirical loss minimization using the \select grouping function and the population loss of the optimal model of client $i$. The right-hand side of Eq.~\eqref{eq:the} bounds this gap with respect to weighted averages of two factors: 1) \textbf{Generalization} (GE) that is related to $N_g$, and 2) \textbf{Distribution Discrepancy} (DD) between the global group distribution $\mathcal{C}_g$ and the local group distribution $\mathcal{D}_g^i$ of client $i$.  \ours{} accounts for these two terms and reduces the gap with shared and group prompts. Specifically, recent studies~\cite{ostapenko2022continual,hayes2020remind} have shown that low-level representations exhibit considerable similarity across different groups, suggesting a relatively small DD term.  In response to these findings, our approach involves inserting shared prompts $P_S$ in early (low-level) layers of ViT and training using all data, thereby maximizing the value of $N_g$ ($g=1, N_g = N$) to effectively reduce the dominant GE term. At higher layers, where the DD dominates the bound due to diverse feature representations~\cite{yu2022tct}, a selection module groups similar data to learn the same group prompt, ensuring $\mathcal{D}^i_g \approx \mathcal{C}^i_g$ (thus reducing DD).
\section{Experiments}
\begin{table*}[t!]
\caption{
Performance comparison of different methods on the CIFAR-100 dataset and the Five dataset. The \textbf{bold} and \underline{underline} highlights represent the best and second-best results, respectively.}
\centering
\resizebox{0.8\linewidth}{!}{%
\begin{tabular}{lcccccc|c c c c}
\hline\hline 
 Datasets &   \multicolumn{6}{c}{ CIFAR-100 (\%)  $\uparrow$ } & \multicolumn{3}{c}{ Five-Dataset (\%) $\uparrow$} \\
\hline
\multirow{2}{*}{ Method } & \multicolumn{2}{c}{ Global} & \multicolumn{2}{c}{ Local } & \multicolumn{2}{c|}{Worst Local } &   \multirow{2}{*}{ Global} &\multirow{2}{*}{Local} & \multirow{2}{*}{Worst Local}  \\
 \cline{2-7} & $s=50$ & $s=10$ & $s=50$ & $s=10$ & $s=50$ & $s=10$  & & & \\
\hline Head-Tune & $76.69$ & $75.35$   & $76.68$  &  $75.36$  & $65.96$ & $55.53$ & $75.09$& $75.09$  & $28.60$    \\
\hline FedVPT~\cite{jia2022visual} & $82.35$ & $80.79$ &  
$85.67$ & $80.79$ &  $72.28$  &  $66.43$  & $80.88$ & $80.88$ & $53.04$\\
\hline FedVPT-D~\cite{jia2022visual} & $85.85$ & $79.49$ & 
$85.85$& $79.49$& $74.83$& $63.38$ & $\underline{\smash{82.09}}$  & $82.09$  & $54.96$  \\
\hline FedMix~\cite{yoon2021fedmix} & $85.65$ & $80.83$  & $85.67$ & $80.82$ & $74.33$ & $69.59$ & $81.07$ & $81.07$ & $42.62$ \\
\hline 
pFedPG~\cite{yang2023efficient} & $84.67$ & $81.19$  & $85.22$ & $\underline{\smash{81.82}}$  & $72.33$ &  $\underline{\smash{70.53}}$ & $72.12$ &  $\underline{\smash{82.48}}$  & $\underline{\smash{56.24}}$ \\
\hline FedEM~\cite{marfoq2021federated} & $81.52$ & $78.41$ & $81.53$  &  $78.40$ & $72.06$ &  $63.87$ & $79.54$ & $79.54$ & $46.50$ \\
\hline 
FedPR~\cite{feng2023learning} & $\underline{\smash{85.92}}$ & $\underline{\smash{81.42}}$   & $\underline{\smash{85.92}}$  & $81.35$ &     $\underline{\smash{75.26}}$ &$63.93$ &  $81.29$ & $81.29$ & $42.12$  \\ 
\hline \ours{} (Ours) &  $\mathbf{86.72}$ &   $\mathbf{84.64}$  & $\mathbf{86.71}$ &  $\mathbf{84.64}$  & $\mathbf{77.56}$  &  $\mathbf{73.85}$ &  $\mathbf{83.40}$ & $\mathbf{83.40}$ & $\mathbf{61.04}$ \\
\hline\hline
\end{tabular}
}
 \label{tab:performance}
 \vspace{-2mm}
\end{table*}
\subsection{Experiment Setup}
\noindent\textbf{Datasets.}
In this section, we introduce datasets with label and feature heterogeneity. \textit{Label heterogeneity}: we demonstrate the effectiveness of our proposed approach for \textit{label heterogeneity} using two datasets. (1) \textit{CIFAR100 dataset} comprises 50,000 training images and 10,000 testing images distributed across 100 classes. (2) \textit{Fivedataset} consists of a sequence of 5 datasets (SVHN, CIFAR10, not-MNIST, Fashion-MNIST, and MNIST) as outlined in the work by~\cite{ebrahimi2020adversarial}.
\textit{Feature heterogeneity}: we also consider feature heterogeneity and follow~\cite{yang2023efficient} to demonstrate the effectiveness of our proposed approach using Office-Caltech10 and DomainNet for \textit{feature heterogeneity}: (1) Office-Caltech10~\cite{saenko2010adapting} is composed of four data domains, including Amazon, DSLR, Webcam, and Caltech. Each domain contains ten classes, with 2,533 images in total. (2) DomainNet~\cite{peng2019moment} consists of 0.6 million images of 345 classes distributed across six domains: clipart, infograph, painting, quickdraw, real, and sketch. Following~\cite{li2021fedbn,yang2023efficient}, we use the top ten most frequent classes to form a sub-dataset for our experiments. \\
\noindent\textbf{Non-IID Settings.}
In this section, we introduce the FL environments and the data partition strategies for various datasets. For clients with \textit{label heterogeneity}, in CIFAR-100, we introduce 100 clients and set a low (hard) client participating ratio ($\gamma$) to 0.05. To introduce data heterogeneity among clients, we apply the “Pathological Partition” ~\cite{oh2021fedbabu,li2022fedtp}. We first sort the data by labels and then allocate data from a specific number of classes ($s$) to each client. Since $s$ is the number of classes each user can have, as $s$ decreases, the degree of data heterogeneity increases. As to the Five Datasets, we distribute the data among 20 clients, with every 4 clients originating from the same dataset. We set the participating rate to $\gamma = 0.1$ and perform training for 50 communication rounds. For conducting clients with \textit{feature heterogeneity}, we follow the newest benchmark~\cite{yang2023efficient} and assign a data domain to a client, indicating the number of clients ($M$) is set as 4 and 6 for Office-Caltech10 and DomainNet, respectively.\\
\noindent\textbf{Implementation Details.}  Following~\cite{yang2023efficient,feng2023learning}, we use ImageNet-21K pre-trained ViT-B-16~\cite{dosovitskiy2020image} as our model because it achieves a good trade-off between performance and efficiency. Since ViT-B-16 is originally trained on images with size 224 and patch size 16, we resize our images to 224 to align with the model's specifications. During the prompt-tuning process, we focus on the shared prompts, group prompts, and the classifier. We use local training epochs (E = 5) for all experiments and set the prompt length to 1 for efficiency. For \textit{label heterogeneity} datasets, we use the last layer's cls token as the input feature for \select and set group number $G$ as $20$ and $5$ for CIFAR-100 and Five-dataset respectively. For \textit{feature heterogeneity}, we use the intermediate ($5$-th) layer's cls token as the input feature for \select because intermediate layers capture the texture-related information~\cite{hayes2020remind} and set $G$ as $4$ and $6$ for Office and DomainNet datasets respectively. We follow other settings in~\cite{yang2023efficient} for consistency. \\
\noindent\textbf{Baseline Methods.}
For \textit{label heterogeneity}, we conduct a comparative analysis of our method against various global-model approaches: VPT~\cite{jia2022visual} that optimized using FedAvg~\cite{mcmahan2017communication} (FedVPT), Head-Tuning~\cite{sun2022exploring}, FedMix~\cite{yoon2021fedmix}, as well as recent FedPR~\cite{feng2023learning}. Additionally, for personalized Federated Learning, we considered pFedPG~\cite{yang2023efficient,shamsian2021personalized} and FedEM~\cite{marfoq2021federated}. For \textit{feature heterogeneity}, we implement the same baseline methods as~\cite{yang2023efficient} as it is the newest benchmark for FL feature heterogeneity.

\subsection{Results}\label{sec:compare}
\begin{table*}[t!]
\centering
\caption{Performance comparison of different methods on the Office-Caltech10 and DomainNet datasets. The \textbf{bold} and \underline{underline} highlights represent the best and second-best results respectively.}
\resizebox{0.9\linewidth}{!}{
\begin{tabular}{lccccc|ccccccc}
\hline \hline \multirow{2}{*}{\begin{tabular}{l} 
Datasets \\
Method
\end{tabular}} & \multicolumn{5}{c}{ Office-Caltech10 (\%)  $\uparrow$ } & \multicolumn{7}{c}{ DomainNet (\%) $\uparrow$} \\
\hline & $A$ & $C$ & $D$ & $W$ & Avg. & $C$ & $I$ & $P$ & $Q$ & $R$ & $S$ & Avg. \\
\hline
Per-FedAvg~\cite{fallah2020personalized} & $91.67$ & $90.22$ & $100.0$ & $100.0$ & $95.47$ & $69.39$ & $48.71$ & $82.07$ & $35.30$ & $90.63$ & $72.56$ & $66.44$ \\
\hline FedRep~\cite{collins2021exploiting} & $91.15$ & $88.44$ & $100.0$ & $100.0$ & $94.90$ & $64.26$ & $38.20$ & $72.86$ & $62.10$ & $82.66$ & $60.11$ & $63.37$ \\
\hline FedVPT~\cite{jia2022visual} & $92.71$ & $84.44$ & $100.0$ & $100.0$ & $94.29$ & $65.59$ & $44.14$ & $76.58$ & $47.30$ & $91.04$ & $60.29$ & $64.16$ \\
\hline FedVPT-D~\cite{jia2022visual} & $91.67$ & $89.33$ & $100.0$ & $100.0$ & $95.25$ & $63.31$ & $43.07$ & $74.80$ & $54.80$ & $87.26$ & $67.15$ & $65.07$ \\
\hline pFedPG~\cite{yang2023efficient} & $94.79$ & $92.44$ & $100.0$ & $100.0$ & $\underline{\smash{96.81}}$ & $73.00$ & $50.08$ & $84.33$ & $60.00$& $94.00$ & $68.41$ & $71.64$ \\
\hline FedPR~\cite{feng2023learning} & $95.31$ & $95.11$ & $100.0$ & $96.61$ & $96.76$ & $88.02$ & $49.16$ & $86.11$ & $70.00$ & $96.06$ &  $83.94$ & $\underline{\smash{78.88}}$ \\
\hline
\ours{} (Ours)  & $\mathbf{95.31}$ & $\mathbf{95.56}$  & $\mathbf{100.0}$ & $\mathbf{100.0}$ & $\mathbf{97.72}$  & $\mathbf{89.54}$ & $\mathbf{52.82}$ & $\mathbf{87.56}$ & $\mathbf{70.00}$ & $\mathbf{96.14}$ & $\mathbf{86.82}$ & $\mathbf{80.48}$ \\
\hline\hline
\end{tabular}
}
\label{tab:feature_shift}
\vspace{-4mm}
\end{table*}
\subsubsection{Label Heterogeneity Results}
We calculate three metrics for evaluation: (1) \textit{Global Accuracy} represents the mean accuracy made by each client overall all testing images regarded as global distribution performance (2) \textit{Local Accuracy} is calculated by averaging the accuracy of each local client on its local test data. (3) \textit{Worst Local Accuracy} demonstrates the worst-performing client result, showcasing the ability to adapt to local data distribution. We follow~\cite{li2022fedtp} to report the averaged testing accuracy for the last 10 consecutive global communication rounds.
\\
\textbf{Overall Performance} The results of CIFAR-100 and Five-datasets are presented in Table~\ref{tab:performance}. Due to the difficulty of datasets, directly leveraging the pre-trained ViT and applying head-tuning cannot achieve good performance. With the help of prompt-tuning, FedVPT serves as a strong baseline compared to recent personalized FL algorithms. This can be attributed to the effectiveness of transformers in mitigating catastrophic forgetting and accelerating convergence in dealing with heterogeneous data~\cite{qu2022rethinking}, which validates our motivation to apply ViT in FL. With the help of our proposed \ours{}, we outperform all baselines on both global and worst-local test accuracy across various degrees of data heterogeneity and datasets by a significant margin.
\\
\textbf{Different Label Heterogeneity Level.} As $s$ decreases, this represents an increase in label heterogeneity. The CIFAR-100 results in Table~\ref{tab:performance} show a performance drop across all methods as heterogeneity increases with the decrease of $s$ from 50 to 10. Fortunately, our \ours{} demonstrates robustness to data heterogeneity, with a smaller performance drop compared with other methods.
\subsubsection{Feature Heterogeneity Results.}
In addition to label heterogeneity, our method can also be applied to datasets with feature shifts. We follow~\cite{yang2023efficient} and report both the performance for each client and the average performance overall for clients. The results of Office-Caltech10 and DomainNet datasets with the presence of domain shifts across clients are presented in Table~\ref{tab:feature_shift}. Among all baselines, our \ours{}
achieves the highest average accuracies on Office-Caltech10 and DomainNet at $97.72\%$ and $80.48\%$ respectively. In addition, \ours{} can also align the global model with different local client data with the worst local performance of each dataset at $95.31\%$ and $52.82\%$ respectively. The results validate our goal to reduce the performance gap between global and local data distribution.

\subsection{Global and Local Performance Trade-off}
We demonstrate our method's capacity to achieve high test accuracy on both global data distribution and local clients' data distributions. As shown in Figure~\ref{fig:100pareto}, we visualize both the global accuracy and the worst local accuracy (client with the worst test accuracy) for the CIFAR-100 dataset with $s=10$. Points situated in the top-right corner indicate better performance in both global data distribution and local clients' data distribution. PFL methods that are finetuned locally tend to overfit on local data distribution. GFL methods, though achieving good generalization, often underperform in local data. Notably, our method shows the greatest capacity for aligning with both global and local data distributions.
\subsubsection{Analysis and Ablation Study}\label{sec:ablation}
In this section, we provide a detailed analysis of each module in our methods. For more analysis, please see the Appendix.
\\
\textbf{Effect of Shared Prompt and Group Prompt.}
To demonstrate the benefits of reducing the terms related to GE and DD as outlined in Theorem~\ref{th:ourtheory2}, we conduct ablation studies on shared and group prompts. As indicated in Table~\ref{tab:ablation}, replacing shared prompts with group prompts, to reduce distribution discrepancy, resulting in a $7\%$ gain in the worst local accuracy. This implies successfully aligning the global model with local data distributions, thereby validating our theoretical motivation. 
Additionally, the group prompts improve global accuracy by around $3\%$ because similar data lead to lower gradient dissimilarity~\cite{kumar2022fine} and benefit the optimization process~\cite{karimireddy2020scaffold}. By additionally adding shared prompts and optimizing prompts with BCD, the global and worst local performance improve by around $2\%$ and $3\%$ respectively. As a result, learning common information with shared prompts benefits the generalization. 

\noindent\textbf{Effect of Block Coordinate Descent.}
In this section, we analyze the effect of our proposed BCD optimization (Section~\ref{sec:bcd}) on the CIFAR-100 dataset, specifically with $s=10$. As shown in Table~\ref{tab:ablation}, the direct addition of group prompts leads to a performance decrease by nearly $3\%$. With the aid of our BCD optimization that iteratively learns shared prompts first and then groups prompts, a significant improvement by $7\%$ is observed. When inverting the BCD order (denoted as BCD-Inv), the result drops significantly. These results validate our BCD optimization approach.
\begin{table}[t!]
\caption{Effect of different prompts and block coordinate descent. We report the results on the CIFAR-100 dataset with $s=10$.}
\centering
\resizebox{0.95\linewidth}{!}{
\begin{tabular}{c|c|c|c|c|c}
\hline\hline
 Share & Group & BCD-Inv &  BCD &  Global $\uparrow$ & Worst Local $\uparrow$\\
   \hline
$\sqrt{ }$ & &  & & $79.49$ & $63.38$\\ 
   \hline
 & $\sqrt{ }$ &  & & $82.51$ &$70.79$\\ 
   \hline
$\sqrt{ }$ & $\sqrt{ }$ &  & & $77.82$ & $62.62$\\ 
   \hline
$\sqrt{ }$ & $\sqrt{ }$  &$\sqrt{ }$  &  & $76.76$ & $62.15$\\ 
   \hline
$\sqrt{ }$ & $\sqrt{ }$ &   &  $\sqrt{ }$ 
 & $\mathbf{84.64}$ & $\mathbf{73.85}$\\ 
\hline\hline 
\end{tabular}
}
 \label{tab:ablation}
\end{table}

\noindent\textbf{Ablation on Clustering Performance.}
We conduct ablation studies on various improvements in learning \select function. We evaluate the clustering accuracy using the CIFAR-100 test dataset, with coarse labels~\cite{krizhevsky2009learning} serving as the ground truth. As shown in Table~\ref{tab:cluster_ab}, the direct application of FedAvg to learn keys results in clustering all data into the same group, resulting in only 5\% accuracy. By adding selection probability $q_g$, \select function successfully learns meaningful keys, as shown in Eq.~\eqref{eq:loss_k}. Introducing momentum update, as shown in Eq.~\eqref{eq:sta_key}, further enhances the clustering performance, achieving approximately 10\% and 5\% improvements, respectively. Additionally, in Fig~\ref{fig:stab}, we plot the mean and standard deviation of the prompt selection numbers over all communication rounds. This demonstrates that our proposed momentum updating improved stability. 
\begin{table}[t!]
\caption{Ablation studies with different improvements on \select. We report the results on the CIFAR-100 dataset with $s=10$.}
\centering
\resizebox{0.7\linewidth}{!}{
\begin{tabular}{c|c|c}
\hline\hline
Dataset &  \multicolumn{2}{c}{ CIFAR-100 (\%)} \\
\hline
Method&  $s=50$ & $s=10$ \\
   \hline
FedAvg & $5.00$ & $5.00$\\ 
   \hline
w/ $q_g$  & $46.09$ & $51.26$ \\ 
   \hline
w/ momentum and $q_g$  & $56.39$ & $56.62$\\ 
\hline\hline 
\end{tabular}
}
\vspace{-3mm}
 \label{tab:cluster_ab}
\end{table}
\begin{table}[t!]
\caption{The size of learnable parameters compared with that of vanilla FedAvg~\cite{mcmahan2017communication}, which trains the whole models.}
\centering
\resizebox{\linewidth}{!}{
\begin{tabular}{c|c|c|c|c}
\hline\hline
Method & \multicolumn{3}{c|}{FedAvg} &  Ours \\
  \hline
Architecture & ResNet-18 & ResNet-50 & ViT-16/B  & ViT-16/B \\

  \hline
$\#$ Parameters $\downarrow$ & $11M$ 
 & $24M$    & $86M$ & $\mathbf{0.1M}$ 
     \\
\hline\hline 
\end{tabular}
}
\vspace{-2mm}
 \label{tab:commu}
\end{table}

\noindent\textbf{Where to attach prompts?}
\begin{figure}[t!]
    \centering
    \includegraphics[width=0.8\linewidth]{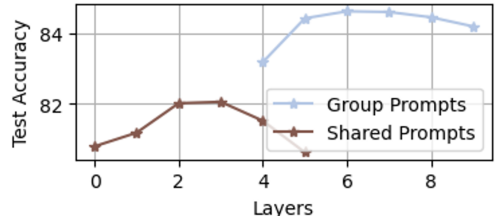}
    \vspace{-4mm}
    \caption{Exploring prompt insertion layers on CIFAR100 ($s$=10). The brown curve represents performance using only shared prompts, while the blue curve illustrates performance with group prompts inserted at varying layers, alongside shared prompts in layer 3.}\label{fig:where}
    \vspace{-4mm}
\end{figure}
In this section, investigate the influence of positions to attach prompts. Here we take CIFAR-100 with $s=10$ as an example and use a heuristic search strategy: (1) We start with examining the position for shared prompts by adding them to the first $U_S$ layers. Our findings, shown in Fig~\ref{fig:where}, inserting shared prompts beyond the $3$-rd layer decreases performance. This is because higher layers suffer more from heterogeneity~\cite{yu2022tct}. As a result, $U_S=3$. (2) Based on the best position for shared prompts, we study the position of group prompts. As depicted in Fig~\ref{fig:where}, adding prompts from $4$-th to $6$-th yields the optimal performance of $84.64\%$. In summary, without specific designs (\ie, our algorithm), training prompts on higher layers prove challenging due to increased heterogeneity. We suggest that prompt tuning in FL should take this aspect into consideration. 

\noindent\textbf{Efficiency.}
In this section, we conduct a comparison between the number of parameters that need training and communication between \ours{} and training a whole network in classical FL (\eg, FedAvg~\cite{mcmahan2017communication}). As shown in Table~\ref{tab:commu}, prompt-tuning requires significantly less trainable parameters compared with traditional FL thus improving the efficiency and saving the communication cost.
\section{Conclusion}
In this work, we demonstrate the significant advancements in FL through the innovative integration of ViT to bridge GFL and PFL. This is achieved through our proposed \ours{}, which introduces a shared and group prompt tuning strategy, enabling the model to adeptly capture both common and group-specific features. The prompt selection module of \ours{} facilitates the training of a global model that can automatically adapt to varied local client data distributions without necessitating local fine-tuning. We employ BCD for effective optimization of the learnable parameters. Theoretically, our approach minimizes the error bound between global and local performances, leading to enhanced model efficacy. Empirical validations against state-of-the-art baselines and thorough ablation studies further underscore \ours{}'s superior performance and efficiency.
{
    \small
    \bibliographystyle{ieeenat_fullname}
    \bibliography{main}
}

\clearpage
\appendix
\onecolumn
\noindent\textbf{Overall pipeline of Appendix.} We first include related work in Section~\ref{sec:rela}. Then the notations and algorithm details of \ours{} are explained in Section~\ref{sec:alg_d}. In Section~\ref{sec:the_d}, we introduce the proof of Theorem~\ref{th:ourtheory2} along with required lemmas. Finally, additional analysis are reported in Section~\ref{sec:add_exp} for a more comprehensive study.
\section{Related Work}\label{sec:rela}
\subsection{Generic Federated Learning}
FedAvg~\cite{mcmahan2017communication} is a standard FL algorithm that involves multiple rounds of local training and global aggregation. Many works focus on improving FedAvg through various aspects: (1) global aggregation methods~\cite{wang2020federated,yurochkin2019bayesian,chen2020fedbe}  replace the weighted average with more dedicated strategies like ensemble and distillation. (2) optimization methods~\cite{karimireddy2020scaffold,li2020federated,karimireddy2020mime} reduce the client drifts~\cite{karimireddy2020scaffold} by correcting local gradients~\cite{karimireddy2020scaffold} or employed regularization toward the global model~\cite{li2020federated} thus achieving a better convergence rate. Recently, FedPR~\cite{feng2023learning}, a prompt-tuning-based GFL method, learns client prompts in the null space of group prompts in the previous round and aggregates them into global prompts, but it may not perform well when the global prompts are not low-rank~\cite{deng2023fairness}. The disadvantage of GFL methods is they are insufficient~\cite{li2020federated} for good performance when dealing with significant data heterogeneity.
\subsection{Personalized Federated Learning}
PFL~\cite{kulkarni2020survey} learns a customized model for each client. To be specific, Fine-tuning methods~\cite{singhal2021federated,oh2021fedbabu} adjust a global or meta-trained model to adapt to local client data by fine-tuning part of the global model. Clustered FL~\cite{ghosh2020efficient,mansour2020three,caldarola2021cluster} clusters clients with similar data distribution, assuming that they can share the same optimal model, however, imposes a heavy computational burden. pFedHN~\cite{shamsian2021personalized} learns a hypernetwork at the server to aggregate clients’ model updates and produce their entire models for the next round. The disadvantage of PFL is overcoming challenges in adapting to new clients and overfitting local data. In this paper, we learn a generalized FL model that not only achieves high accuracy on the global distribution but is also capable of aligning with different local distributions without local adaptation.
\subsection{Parameter Efficient Tuning}
Parameter efficient tuning (PET)~\cite{ding2023parameter}, initially proposed for text models, enables easier access and usage of pre-trained models by reducing the memory cost needed to conduct fine-tuning due to fewer computed gradients. PETuning techniques, including methods like Adapter Tuning~\cite{li2022fedtp}, LoRA~\cite{yi2023fedlora}, Prompt-Tuning~\cite{feng2023learning}, and Head-Tuning~\cite{yang2023efficient}, freeze most parameters of pre-trained models and update only a few additional parameters or a part of the original model parameters for downstream tasks. This paper focuses on prompt tuning due to no need to modify anything inside the neural network~\cite{liu2023pre,lester2021power} and Visual Prompt Tuning (VPT)~\cite{jia2022visual} has been established as an efficient and effective PETuning method for adapting large-scale ViT models to vision tasks. Recent studies on VPT have been conducted in fields like continual learning~\cite{wang2022dualprompt,wang2022learning} and multi-modality learning~\cite{khattak2023maple,li2023efficient}. Although these advancements have shown progress in various visual tasks, prompt tuning remains predominantly limited to centralized systems. The effectiveness of prompt tuning in a distributed framework has yet to be thoroughly investigated.

\section{Algorithm details}\label{sec:alg_d}
\subsection{Notation}
For convenience, we summarize the notations used in this paper in Table~\ref{tab:Notation}. 
\\
\begin{table}[htbp]
\caption{ Main notations used in this paper.}
\begin{center}
\resizebox{0.6\columnwidth}{!}{%
\begin{tabular}{r c p{7cm} }
\toprule
\multicolumn{3}{c}{\underline{Basic Variables}} \\
$M$ & $\triangleq$ & Number of clients\\
$\mathcal{X}$ & $\triangleq$ & Input space \\
$\mathcal{Y}$ & $\triangleq$ & Label space \\
$\mathcal{D}_i$ & $\triangleq$ & Data distribution on client $i$ and $\mathcal{D}_i$ is on $\mathcal{X}\times \mathcal{Y}$\\
$N = \sum_{i=1}^{M}N_i$ & $\triangleq$ & $N_i$ is number of data samples at client $i$ and $N$ is the number of data on all clients \\
$\{x_i,y_i\} \sim \mathcal{D}_i $ & $\triangleq$ & Data sample $(x_i,y_i)$ located on participating client $i$ is made of i.i.d sampling from $\mathcal{D}_i$ \\
\multicolumn{3}{c}{\underline{Function Variables}} \\
$\ell$ & $\triangleq$ & loss function\\  
$\select$ & $\triangleq$ & Prompt Selection function\\  
$[G]$, $G$ & $\triangleq$ & Group set, number of groups $G$\\ 
$\mathcal{D}^i_g$ & $\triangleq$ & Data distribution on client $i$ from group $g$ \\
$N^i_g$ & $\triangleq$ & Number of data samples at client $i$ from group $g$ \\
$N_g = \sum_{i=1}^{M}N^i_g$ & $\triangleq$ & Number of data samples from group $g$ \\

$K$, $k_g$ & $\triangleq$ & Keys set, key of $g$-th group $|\mathcal{G}|$\\  
$P_G$ & $\triangleq$ & Weight of group prompts\\  
$P_S$ & $\triangleq$ & Weight of shared prompts\\  
$W_C$ & $\triangleq$ & Weight of classifier\\ 
$h_{\theta}$ & $\triangleq$ &  Pretrained foundation model\\ 
$cls$ & $\triangleq$ &  Classification token\\ 
$E$ & $\triangleq$ &  Image patch embeddings\\ 
$Z$ & $\triangleq$ &  Prompt features embeddings\\ 
\bottomrule
\end{tabular}
}
\end{center}
\label{tab:Notation}
\end{table}

\subsection{Training Algorithm}
We present more details in the training phase of \ours{}. We provide a detailed illustration of the training process for our proposed method in Algorithm~\ref{alg:training} based on the commonly used FedAvg~\cite{mcmahan2017communication} scheme: during each communication round, the clients engage in local training using the global model received from the server, and the server aggregates the shared parameters from the clients to update the global model. Notably, \ours{} can be combined with other FL methods. Particularly, in \ours{}'s training, the local model parameters (\ie, $P_S$, $W_C$, $K$, and $P_G$) are sent back to the server for aggregation. For model parameters, the aggregation weight of client $i$ is determined by $\alpha_i = {N_i}/{N}$. Regarding the aggregation of keys, we initially calculate the selection quantity of a group $g$ on client $i$ at round $t$ denoted as $N^{i,t}_g$. Then, the aggregation weight for a group key is computed as ${N^{i,t}_g}/{\sum_i N^{i,t}_g}$. At last, the momentum aggregation is applied to the keys and group prompts.
\subsection{Inference Algorithm}
We present more detials in the inference phase of \ours{}. The inference procedure of \ours{} in Algorithm~\ref{alg:inf}. Given a sample $x$, we first use the \select function (see Eq.~\eqref{eq:select}) to determine its group membership $g = \select(x)$. Then, the shared prompt $P_S$ and corresponding group prompts $p_g$ are inserted into the model to achieve sample-level adaptation for inference. When performing tests on new clients, the frequencies of selected group prompts can automatically adjust by \select function (shown in Fig.~\ref{fig:main} (c)), ensuring our model aligns with their local data distributions. 
\begin{algorithm}[t!]
\textbf{Server Input:} Initial weights $W = \{W_C,P_G,P_S\}$, prompt selection  module $\select$ with learnable keys $K =\{k_{g}\}_{g=1}^{G}$, number of participating clients in each round $m = \gamma\times M$, number of communication rounds $T$, client data ratio set $\{\alpha_i\}_{i=1}^{M}$, accumulated selection quantity $\{v_g^i\}_{i=1}^{M}, g\in[G]$.\\
\textbf{Client $i$'s Input:} Pre-trained Transformer $h_{\theta}$, training data   ${(x_i,y_i) \sim D_i} $ for client $i$,  learning rate $\eta$, number of local training steps $E$.
 \begin{algorithmic}[1]
 \State For $t = 1 \rightarrow T$ rounds, sample $m$ clients and execute \textbf{Procedure A} and \textbf{Procedure B} iteratively.
 \Procedure{\textbf{A}: ClientUpdate}{$i$}
 \State $W_i \leftarrow W$   \algorithmiccomment{Initialize with global model} 

 \State $K_i \leftarrow K$   \algorithmiccomment{Initialize with global keys} 
 
\State Apply Block Coordinate Descent (Algorithm~\ref{alg:BCD}) to update the parameter $W_i$ and $K_i$.
\State Count the group selection quantity $\{N_g^{i}\}_{g\in G}$
\State Send updated $W_i$, $K_i$ and $\{N_g^{i}\}_{g\in G}$ to \Call{ServerExecute}{}
\EndProcedure
\Procedure{\textbf{B}: ServerExecute}{$t$}
\State Receive local models' parameters from \Call{ClientUpdate}{}
\State  $[W^t_C,P^t_S,P^t_G] \leftarrow \sum_{i=1}^{m}\alpha_i [W_C^{i,t},P_S^{i,t},P^{i,t}_G]$ \algorithmiccomment{Parameter Aggregation} 
\For{$g = 1\rightarrow G$}

\State  $k^t_g \leftarrow \sum_{i=1}^{m}\frac{N^{i,t}_g}{\sum_{i \in [m]} N^{i,t}_g} k_g^{i,t}$ 
\algorithmiccomment{Key Aggregation} 
\State $v^t_g = v^{t-1}_g + \sum_{i=1}^{m}N_g^{i,t}$
\algorithmiccomment{Accumulate Group Number} 
\EndFor
\State Apply Momentum Aggregation Eq.~\ref{eq:sta_key}.
\State ($\hat{k}_g^{0}, \hat{p}_g^{0} = k_g^{0}, k_g^{0} $ if $t =0$)  
\State  $\hat{k}_g^t = \alpha_k \hat{k}_g^{t-1} + (1-\alpha_k) k_g^t $
 \State   $\hat{p}_g^t = \alpha_g \hat{p}_g^{t-1} + (1-\alpha_g) p_g^t , \quad g\in[G]$

\State broadcast parameters to \Call{ClientUpdate}{}
\EndProcedure
 \end{algorithmic}
 \caption{\ours{} Training Algorithm}
 \label{alg:training}
\end{algorithm}

\begin{algorithm}[t!]
\textbf{Input:} Pre-trained Transformer $h_{\theta}$ with 
$U$ layers, layers set $\mathcal{U}_S$ to insert shared prompts and layers set  $\mathcal{U}_G$ to insert group prompts, trainable weights $W = \{W_C,P_G,P_S\}$, Prompt Selection Module \select with orthogonal keys $K =\{k_{g}\}_{g=1}^{G}$, a single test sample $x_n$.
 \begin{algorithmic}[1]
\State Extract representation $h_{\theta}(x_n)$ with pre-trained model $h_{\theta}$
\State $g \leftarrow$ $ \argmax_{g\in \mathcal{G}}{\rm cos}(h_{\theta}(x_n),b_g)\ $ \algorithmiccomment{Obtain group ID} 
\State Select corresponding group prompt $p_g$
 \State Encode $x_n$ into $E_0$ \algorithmiccomment{Encode image into patch embedding (Section~\ref{sec:vpt_f})} 
 \For{$i \rightarrow U$}
     \If{$i\in  \mathcal{U}_S$}
      \State $\left[cls_1,Z^S_1,E_1\right]  =\textcolor{blue}{h_{\theta}^{(i)}}\left(\left[\textcolor{blue}{cls_0}, \textcolor{red}{P^{(i)}_S}, E_0\right]\right).$ \algorithmiccomment{Insert shared prompts } 
      \ElsIf{$i<u$}
 \State $\left[cls_i,Z^S_i,E_i\right]  =\textcolor{blue}{h_{\theta}^{(i)}}\left(\left[\textcolor{blue}{cls_{i-1}}, Z^S_{i-1}, E_{i-1}\right]\right).$
 \ElsIf{$i \in \mathcal{U}_G$}
 \State $\left[cls_u,Z^g_u,Z_u^{S},E_u\right]  =\textcolor{blue}{{h_{\theta}^{(u)}}}\left(\left[\textcolor{blue}{cls_{u-1}},\textcolor{red}{p^{(i)}_{g}},Z_{u-1}^S, E_{u-1}\right]\right).$ \algorithmiccomment{Insert group prompt} 
\Else 
 \State $\left[cls_i,Z^g_{i}, Z^{S}_i,E_i\right]  =\textcolor{blue}{h_{\theta}^{(i)}}\left(\left[\textcolor{blue}{cls_{i-1}}, Z^g_{i-1}, Z^S_{i-1}, E_{i-1}\right]\right).$
\EndIf
 \EndFor
 \State $y^* \leftarrow f(W_C)$ 
 \algorithmiccomment{Final Prediction} 

\State \textbf{return} final logits prediction $y^*$
 \end{algorithmic}
  \caption{\ours{} Inference Algorithm.}
 \label{alg:inf}
\end{algorithm}
\section{Technical details and Omitted proofs}\label{sec:the_d}
\subsection{Settings and Definitions} 
First, we formally set up some general notation: For a distribution $\Dc$ with support $ (\mathcal{X},\mathcal{Y})$ and a non-negative loss function $\ell : \mathcal{X} \times \mathcal{Y} \rightarrow \mathbb{R}^{+}$ denote the population risk of a hypothesis $h:\Xc\rightarrow\Yc$ as follows:
$$\mathcal{L}_{\Dc}(h) = \mathbb{E}_{(X,Y)\sim \Dc}[\ell (h(X),Y)]\,.$$
Let $\Hc$ represent a hypothesis class and denote the hypothesis $\hat h$ minimizing the empirical risk as
$$
\widehat{h} = \arg\min_{h\in \mathcal{H}}\mathcal{L}_{\widehat{\mathcal{D}}}(h),
$$
where we denote $\hat\Dc$ the empirical distribution of samples drawn iid from $\Dc.$ 
We will also denote $\mathfrak{R}_{\Dc,n}(\mathcal{H})$ the Rademacher complexity of the hypothesis class $\Hc$ over the distribution $\Dc$ with $n$ samples.
Furthermore, we define the distribution mismatch between two distributions $\mathcal{D}_1$ and $\mathcal{D}_2$ as 
\begin{align}\label{eq:dismatch}
\operatorname{disc}_{\mathcal{H}}\left(\mathcal{D}_1, \mathcal{D}_2\right)&=\max _{h \in \mathcal{H}}\left|\mathcal{L}_{\mathcal{D}_1}(h)-\mathcal{L}_{\mathcal{D}_2}(h)\right|. 
\end{align}

Next, we formally introduce the statistical setting of our analytical investigation of the impact of group-aware hypothesis inference in a multi-client setting. For clients $i\in[M]$, let $\Dc_i$  denote their corresponding distributions of data pairs $(X,Y)$. We assume that each local distribution $\Dc_i$ is a mixture of group-specific distributions $\Dc_g^i, g\in[G]$ for $G$. Concretely,
\begin{align}\label{eq:Dc_i}
\Dc_i=\sum_{g\in[G]}\pi_g^i\,\mathcal{D}^i_g,
\end{align}
with mixing probability vector $[\pi_1^i,\ldots, \pi_G^i]$. 
This  also induces a probability distribution $\Cc_g$ of data belonging to group $g$ as follows:
\begin{align}\label{eq:group_dis}
      \mathcal{C}_g = \sum_{i\in[M]} \pi_g^i \,\mathcal{D}^i_g.
 \end{align}
Following \cite{hugeneralization}, we refer to the distribution $\Cc_g$ as the participated client's data distribution for the $g$-th group. In our setting, the group assignment formalized above corresponds to the execution of the function $\select: \mathcal{X} \rightarrow [G]$ assigning data to different groups. 

We will also consider the empirical versions of the above distributions. Formally, let $\hat\Dc_i$ the induced local empirical distribution of client $i$ by sampling $N_g^i$ iid samples from each $\Dc_g^i$. Further, let $N_i=\sum_{g\in[G]}N_g^i$ denote the total number of samples per client and $N_g = \sum_{i=1}^{M}N^i_g$ denote the total number of data samples from group $g$. Onwards, we assume that the mixture weights in \eqref{eq:Dc_i} are set as $\pi_g^i=N_g^i/N_i$. Thus, we also define the empirical distribution $\hat\Cc_g$ of each group as $\hat\Cc_g=\sum_{i\in[M]}\pi_g^i\hat\Dc_g^i=\sum_{i\in[M]}\frac{N_g^i}{N_i}\hat\Dc_g^i$. Finally, given a set $\{h_1,\ldots,h_G\}$ of $G$ hypotheses $h_g\in\Hc,g\in[G]$ one corresponding to each group, we denote $h_\select$ the group-aware (data-point dependent) hypothesis determined by the $\select$ function. In other words, $h_\select=\{h_1,\ldots,h_G\}_{\select(x)}$ denotes a hypothesis that when acting on datapoint $x$ returns $h_{\select(x)}$, for a fixed function $\select: \mathcal{X} \rightarrow [G]$.

\subsection{Lemmas}\label{app:lem1}
In this section, we cover some technical lemmas which are useful for proving our main result. Lemma~\ref{lem:1} below splits the participated clients' distribution risk into the risks of each individual's group. 
\begin{lemma}[Split the distribution risk]\label{lem:1}
{Let fixed function $\select: \mathcal{X}\rightarrow \{1,...,G\}$. For the group-aware hypothesis $h_{\select}$ that selects among hypotheses $\{h_1,\ldots,h_G\}$, it holds for any client $i\in [M]$ that}
$$\mathcal{L}_{\mathcal{D}_i}(h_{\select}) = \sum_{g=1}^{G}\frac{N^i_g}{N_i} \mathcal{L}_{\mathcal{D}^i_g}(h_g).$$
\end{lemma}
\begin{proof} By definition of the population risk:
\begin{align*}
    \mathcal{L}_{\mathcal{D}_i}(h_{\select})
    &= \mathbb{E}_{(x,y)\sim \mathcal{D}_i}[\ell(h_{\select},x,y)].
\end{align*}
The desired follows from this by recalling the decomposition in \eqref{eq:Dc_i} with weights mixing weights $\pi_g^i=\frac{N^i_g}{N_i}.$
\end{proof}

The next lemma is useful to derive global and local performance gap.
\begin{lemma}[Bound on the generalization error]\label{lem:gener} {Assume the loss is bounded in $[0,1]$ and fix any client $i \in [M]$. Then with probability at least $1-\delta$ over the training set,}
\begin{align}
&\sum_{g=1}^{G} \frac{N^i_g}{N_i} \min _{h_g \in H} \mathcal{L}_{\mathcal{C}_g}(h_g)
    - \sum_{g=1}^{G} \frac{N^i_g}{N_i} \min _{h_g \in H} \mathcal{L}_{\widehat{\mathcal{C}}_g}(h_g) 
\leq 2 \sqrt{\frac{\log \frac{1}{\delta}}{N_i}}+  \sum_{g=1}^G  \frac{N^i_g}{N_i} \Re_{\mathcal{C}_g, N_g}(\mathcal{H}).
\end{align}

\end{lemma}
\begin{proof} 
For any set of real numbers $a_1,...,a_q$ and $b_1,...,b_1$ observe that $\min _i a_i \leq \max _i a_i$ and $\min _i b_i = -\max _i -b_i$, we get
\begin{align*}
&\min _i a_i-\min _i b_i \leq \max _i\left(a_i-b_i\right).
\end{align*}
Using this it holds that
\begin{align*}
&\sum_{g=1}^{G} \frac{N^i_g}{N_i} \min _{h_g \in H} \mathcal{L}_{\mathcal{C}_g}(h_g)
    - \sum_{g=1}^{G} \frac{N^i_g}{N_i} \min _{h_g \in H} \mathcal{L}_{\widehat{\mathcal{C}}_g}(h_g)  \leq \sum^{G}_{g=1}\frac{N_g^i}{N_i}  \max_{h_g} \left (\mathcal{L}_{\mathcal{C}_g}(h_g) - \mathcal{L}_{\hat{\mathcal{C}}_g}(h_g) \right )\,.
\end{align*}
Since the loss is bounded in $[0,1]$, changing one sample changes the above term by at most one. Thus, by the McDiarmid’s inequality, with probability at least $1-\delta$, 
\begin{align} \label{eq:mcD}
& \sum^{G}_{g=1}\frac{N_g^i}{N_i}  \max_{h_g} \left (\mathcal{L}_{\mathcal{C}_g}(h_g) - \mathcal{L}_{\hat{\mathcal{C}}_g}(h_g) \right )   \leq \frac{1}{N_i} \mathbb{E}\left[\sum^{G}_{g=1}N_g^i \max_{h_g} \left (\mathcal{L}_{\mathcal{C}_g}(h_g) - \mathcal{L}_{\hat{\mathcal{C}}_g}(h_g) \right )\right] +\sqrt{\frac{1}{N_i}\log \frac{1}{\delta}} . 
\end{align}
Moreover, note that:
\begin{align}
    \mathbb{E}\left[ \sum^{G}_{g=1}\frac{N_g^i}{N_i}\max_{h_g}  \left (\mathcal{L}_{\mathcal{C}_g}(h_g) - \mathcal{L}_{\hat{\mathcal{C}}_g}(h_g) \right )\right] \notag  & =  \sum^{G}_{g=1} \frac{ N_g^i}{N_i} \mathbb{E}\left[ \max_{h_g}  (\mathcal{L}_{\mathcal{C}_g}(h_g) - \mathcal{L}_{\hat{\mathcal{C}}_g}(h_g) )\right] \\ & \leq  2 \sum_{g=1}^G \frac{N^i_g}{N_i}  \Re_{\mathcal{C}_g, N_g}(\mathcal{H}). \label{eq:radmache}
\end{align}
Combining Eq.~\eqref{eq:mcD} and Eq.~\eqref{eq:radmache}, completes the proof.
\end{proof}
\subsection{Proof of Theorem 1}\label{app:the2}
We are now ready to prove Theorem~\ref{th:ourtheory2}.
\begin{proof} 
Given Lemma~\ref{lem:1}, we split the distribution risk. Thus, the global-to-local performance gap becomes
$$\sum_{g=1}^{G} \frac{N^i_g}{N_i}  \mathcal{L}_{\widehat{\mathcal{D}}^i_g}(\widehat{h}_g)-\min _{h \in H} \mathcal{L}_{\mathcal{D}_i}(h),$$
where $\widehat{h}_g = \arg\min_{h\in \mathcal{H}}\mathcal{L}_{\widehat{\mathcal{C}}_g}(h)$ denote the empirical model for data group $g$. Next, observe from \eqref{eq:Dc_i} that 
\begin{align}
    \min _{h \in H} \mathcal{L}_{\mathcal{D}_i}(h) &= 
    \min_{h\in\Hc}\sum_{g=1}^{G} \frac{N^i_g}{N_i}  \mathcal{L}_{\mathcal{D}_g^i}(h)\,\nn
    \\
    &\geq \sum_{g=1}^{G} \frac{N^i_g}{N_i} \min _{h_g \in H} \mathcal{L}_{\mathcal{D}_g^i}(h_g)\,.\nn
\end{align}
 Denote $h_{gi}^{\star} = \min _{h_g \in H} \mathcal{L}_{\mathcal{D}_g^i}(h_g)$, we then get the following:
\begin{align}
      \sum_{g=1}^{G} \frac{N^i_g}{N_i}  \mathcal{L}_{\widehat{\mathcal{D}}^i_g}(\widehat{h}_g)- \min _{h \in H} \mathcal{L}_{\mathcal{D}_i}(h)
    & \leq \sum_{g=1}^{G} \frac{N^i_g}{N_i} \mathcal{L}_{\widehat{\mathcal{D}}^i_g}(\widehat{h}_g)-\sum_{g=1}^{G} \frac{N^i_g}{N_i} \min _{h_g \in H} \mathcal{L}_{\mathcal{D}_g^i}(h_g) \notag \\
    & = \sum_{g=1}^{G} \frac{N^i_g}{N_i}  \mathcal{L}_{\widehat{\mathcal{D}}^i_g}(\widehat{h}_g)- \sum_{g=1}^{G} \frac{N^i_g}{N_i} \left[ \mathcal{L}_{\mathcal{C}_g}(h_{gi}^{\star}) + \mathcal{L}_{\mathcal{D}_g^i}(h_{gi}^{\star}) - \mathcal{L}_{\mathcal{C}_g}(h_{gi}^{\star})  \right]
     \notag \\
     & = \sum_{g=1}^{G} \frac{N^i_g}{N_i}  \mathcal{L}_{\widehat{\mathcal{D}}^i_g}(\widehat{h}_g)- \sum_{g=1}^{G} \frac{N^i_g}{N_i} \left[ \mathcal{L}_{\mathcal{C}_g}(h_{gi}^{\star}) + \mathcal{L}_{\mathcal{D}_g^i}(h_{gi}^{\star}) -\mathcal{L}_{\mathcal{C}_g}(h_{gi}^{\star})  \right]
     \notag \\
     &  - \sum_{g=1}^{G} \frac{N^i_g}{N_i} \mathcal{L}_{\widehat{\mathcal{C}}_g}(\widehat{h}_g)
    +  \sum_{g=1}^{G} \frac{N^i_g}{N_i} \mathcal{L}_{\widehat{\mathcal{C}}_g}(\widehat{h}_g) 
     \notag \\
    & = \sum_{g=1}^{G} \frac{N^i_g}{N_i}  \mathcal{L}_{\widehat{\mathcal{D}}^i_g}(\widehat{h}_g) 
    - \sum_{g=1}^{G} \frac{N^i_g}{N_i} \mathcal{L}_{\widehat{\mathcal{C}}_g}(\widehat{h}_g)
    \notag \\
    & 
    + \sum_{g=1}^{G} \frac{N^i_g}{N_i} \left ( \mathcal{L}_{\mathcal{C}_g}(h_{gi}^{\star}) - \mathcal{L}_{\mathcal{D}_g^i}(h_{gi}^{\star}) \right )   \notag \\
     &   +  \sum_{g=1}^{G} \frac{N^i_g}{N_i} \mathcal{L}_{\widehat{\mathcal{C}}_g}(\widehat{h}_g)  -  \sum_{g=1}^{G} \frac{N^i_g}{N_i} \mathcal{L}_{\mathcal{C}_g}(h_{gi}^{\star})
    \label{eq:c3.1}
\end{align}
Observing that $\sum_{g=1}^{G} \frac{N^i_g}{N_i} \min _{h_g \in H} \mathcal{L}_{\mathcal{C}_g}(h_{g}) \leq  \sum_{g=1}^{G} \frac{N^i_g}{N_i} \mathcal{L}_{\mathcal{C}_g}(h_{gi}^{\star})$ we get
\begin{align*}
    \eqref{eq:c3.1} & \leq  \sum_{g=1}^{G} \frac{N^i_g}{N_i} \max_{h_g \in H} \left | \mathcal{L}_{\mathcal{D}^i_g}(h_g)  
    - \mathcal{L}_{\mathcal{C}_g}(h_g) \right |  +  \sum_{g=1}^{G} \frac{N^i_g}{N_i} \max_{h_g \in H} \left |\mathcal{L}_{\widehat{\mathcal{D}}_g^i}(h_g) - \mathcal{L}_{\widehat{\mathcal{C}}_g}(h_g)  
    \right |
    \notag\\
    &+  \sum_{g=1}^{G} \frac{N^i_g}{N_i} \min _{h_g \in H} \mathcal{L}_{\widehat{\mathcal{C}}_g}(h_g) 
    - \sum_{g=1}^{G} \frac{N^i_g}{N_i} \min _{h_g \in H} \mathcal{L}_{\mathcal{C}_g}(h_g)\,.
   \notag \\
\end{align*}
Then, combining lemma~\ref{lem:gener}, absolute homogeneity of Rademacher complexity, and the definition of the discrepancy in Eq.~\eqref{eq:dismatch}, we will have with probability $1-\delta$,
\begin{align}
      \sum_{g=1}^{G} \frac{N^i_g}{N_i}  \mathcal{L}_{\widehat{\mathcal{D}}^i_g}(\widehat{h}_g)- \min _{h \in H} \mathcal{L}_{\mathcal{D}_i}(h)  \leq \sqrt{\frac{\log \frac{1}{\delta}}{N_i}}+ 2 \sum_{g=1}^G \frac{N^i_g}{N_i} \Re_{\mathcal{C}_g, N_g}(\mathcal{H}) +  \sum_{g=1}^{G} \frac{N^i_g}{N_i} \left(\operatorname{disc}(\mathcal{D}_g^i, \mathcal{C}_{g}) + \operatorname{disc}(\widehat{\mathcal{D}}_g^i, \widehat{\mathcal{C}}_{g}) \right)
\end{align}
When the VC dimension of Hypothesis class $\mathcal{H}$ is $d$, then we can obtain:
\begin{align}
     \sum_{g=1}^{G} \frac{N^i_g}{N_i}  \mathcal{L}_{\widehat{\mathcal{D}}^i_g}(\widehat{h}_g)- \min _{h \in H} \mathcal{L}_{\mathcal{D}_i}(h) \leq \sqrt{\frac{\log \frac{1}{\delta}}{N_i}}+ 2 \sum_{g=1}^G \frac{N^i_g}{N_i} 
\sqrt{\frac{2d}{N_g}log(\frac{eN_g}{d})}  +  \sum_{g=1}^{G} \frac{N^i_g}{N_i} \left(\operatorname{disc}(\mathcal{D}_g^i, \mathcal{C}_{g}) + \operatorname{disc}(\widehat{\mathcal{D}}_g^i, \widehat{\mathcal{C}}_{g}) \right)
\end{align}
This completes the proof of Theorem~\ref{th:ourtheory2}. 

\end{proof}
\subsection{Performance Gap on Real Distribution}\label{app:the_g}
In this section, we follow the idea in ~\cite{bao2023optimizing} to give the gap between the population loss of the global model found by empirical loss minimization using the \select grouping function and the population loss of the optimal model of client $i$. Different from ~\cite{bao2023optimizing} that focus on clustering clients, here we focus on clustering data into groups: $\mathcal{L}_{\mathcal{D}_i}(\widehat{h}_{\select}) -\min _{h \in H} \mathcal{L}_{\mathcal{D}_i}(h)$. Different from theirs based on clustering clients into non-overlapping coalitions, we focus on learning parameters for each data group, allowing each client to benefit from knowledge distilled from all other clients’ datasets~\cite{marfoq2021federated}.  
We present the theorem below:
\begin{theorem} Assume the loss function $\ell$ is bounded in $[0,1]$ and the function $\select$ is a data grouping method. Let $\mathfrak{R}_{\mathcal{D},m}(\mathcal{H})$ represent the Rademacher complexity of the hypothesis class $\mathcal{H}$ over the distribution $\mathcal{D}$ with $m$ samples. Then, with a probability of at least $1-\delta$ over the training set,
\begin{align}
\mathcal{L}_{\mathcal{D}_i}(\widehat{h}_{\select})   -\min _{h \in H} \mathcal{L}_{\mathcal{D}_i}(h)
 \leq  2\sqrt{\frac{\log \frac{1}{\delta}}{N_i}} & + 4\sum_{i=1}^G \frac{N^i_g}{N_i}  \Re_{\mathcal{C}_g, N_g}(\mathcal{H}) 
    +  \sum_{g=1}^{G} \frac{N^i_g}{N_i} \left(2\operatorname{disc}(\mathcal{D}_g^i, \mathcal{C}_{g}) \right), 
\end{align}  
where $\operatorname{disc}_{\mathcal{H}}\left(\mathcal{D}_1, \mathcal{D}_2\right)=\max _{h \in \mathcal{H}}\left|\mathcal{L}_{\mathcal{D}_1}(h)-\mathcal{L}_{\mathcal{D}_2}(h)\right|$.
\end{theorem}
The obtained theorem also suggests that tuning $G$ is important in achieving optimal performance, which agrees with theorem~\ref{th:ourtheory2}.
\begin{proof}
\begin{align}
\mathcal{L}_{\mathcal{D}^i}(\widehat{h}_g)   &   \leq \sum_{g=1}^{G} \frac{N^i_g}{N_i} \mathcal{L}_{\mathcal{C}_g}(\widehat{h}_g) + \mathcal{L}_{\mathcal{D}^i}(\widehat{h}_g) -  \sum_{g=1}^{G} \frac{N^i_g}{N_i} \mathcal{L}_{\mathcal{C}_g}(\widehat{h}_g)   \notag \\
 & = \sum_{g=1}^{G} \frac{N^i_g}{N_i} \mathcal{L}_{\mathcal{C}_g}(\widehat{h}_g) + \sum_{g=1}^{G} \frac{N^i_g}{N_i}  \mathcal{L}_{\mathcal{D}^i_g}(\widehat{h}_g)  \notag -  \sum_{g=1}^{G} \frac{N^i_g}{N_i} \mathcal{L}_{\mathcal{C}_g}(\widehat{h}_g)  \notag \\  
  & = \sum_{g=1}^{G} \frac{N^i_g}{N_i} \mathcal{L}_{\mathcal{C}_g}(\widehat{h}_g) + \sum_{g=1}^{G} \frac{N^i_g}{N_i} \operatorname{disc}(\mathcal{D}_g^i, \mathcal{C}_{g}) \notag \\
   & \leq \sum_{g=1}^{G} \frac{N^i_g}{N_i} \mathcal{L}_{\widehat{\mathcal{C}}_g}(\widehat{h}_g) + \sum_{g=1}^{G} \frac{N^i_g}{N_i} \operatorname{disc}(\mathcal{D}_g^i, \mathcal{C}_{g})    + \sum_{g=1}^{G} \frac{N^i_g}{N_i} \left( \mathcal{L}_{\mathcal{C}_g}(\widehat{h}_g)  - \mathcal{L}_{\widehat{\mathcal{C}}_g}(\widehat{h}_g) \right)  \notag \\ 
    & \leq  \sum_{g=1}^{G} \frac{N^i_g}{N_i} \mathcal{L}_{\widehat{\mathcal{C}}_g}(\widehat{h}_g) + \sum_{g=1}^{G} \frac{N^i_g}{N_i} \operatorname{disc}(\mathcal{D}_g^i, \mathcal{C}_{g})     +  \sum^{G}_{g=1}\frac{N_g^i}{N_i}  \max_{h_g} \left (\mathcal{L}_{\mathcal{C}_g}(h_g) - \mathcal{L}_{\hat{\mathcal{C}}_g}(h_g) \right ) \notag \\ 
    & \leq  \sum_{g=1}^{G} \frac{N^i_g}{N_i} \mathcal{L}_{\mathcal{C}_g}(h^{\star}_{gi}) + \sum_{g=1}^{G} \frac{N^i_g}{N_i} \operatorname{disc}(\mathcal{D}_g^i, \mathcal{C}_{g}) \notag \\ 
    &+   \left |  \sum_{g=1}^{G} \frac{N^i_g}{N_i} \left( \mathcal{L}_{\widehat{\mathcal{C}}_g}(h^{\star}_{gi}) - \mathcal{L}_{\mathcal{C}_g}(h^{\star}_{gi}) \right ) \right |   +  \sum^{G}_{g=1}\frac{N_g^i}{N_i}  \max_{h_g} \left (\mathcal{L}_{\mathcal{C}_g}(h_g) - \mathcal{L}_{\hat{\mathcal{C}}_g}(h_g) \right )   \notag \\ 
     & \leq   \sum_{g=1}^{G} \frac{N^i_g}{N_i} \mathcal{L}_{\mathcal{C}_g}(h^{\star}_{gi}) + \sum_{g=1}^{G} \frac{N^i_g}{N_i} \operatorname{disc}(\mathcal{D}_g^i, \mathcal{C}_{g})    + 2 \sum^{G}_{g=1}\frac{N_g^i}{N_i}  \max_{h_g} \left (\mathcal{L}_{\mathcal{C}_g}(h_g) - \mathcal{L}_{\hat{\mathcal{C}}_g}(h_g) \right )   \notag \\
      & \leq  \sum_{g=1}^{G} \frac{N^i_g}{N_i} \mathcal{L}_{\mathcal{D}^i_g}(h^{\star}_{gi}) + \sum_{g=1}^{G} \frac{N^i_g}{N_i} \operatorname{disc}(\mathcal{D}_g^i, \mathcal{C}_{g})  +  2 \sum^{G}_{g=1}\frac{N_g^i}{N_i}  \max_{h_g} \left (\mathcal{L}_{\mathcal{C}_g}(h_g) - \mathcal{L}_{\hat{\mathcal{C}}_g}(h_g) \right )  \notag \\
      &   + \sum_{g=1}^{G} \frac{N^i_g}{N_i}  \left ( \mathcal{L}_{\mathcal{C}_g}(h^{\star}_{gi}) -\mathcal{L}_{\mathcal{D}^i_g}(h^{\star}_{gi}) \right )  \notag \\
       & \leq \sum_{g=1}^{G} \frac{N^i_g}{N_i} \mathcal{L}_{\mathcal{D}^i_g}(h^{\star}_{gi})   + 2 \sum_{g=1}^{G} \frac{N^i_g}{N_i} \operatorname{disc}(\mathcal{D}_g^i, \mathcal{C}_{g})  +  2 \sum^{G}_{g=1}\frac{N_g^i}{N_i}  \max_{h_g} \left (\mathcal{L}_{\mathcal{C}_g}(h_g) - \mathcal{L}_{\hat{\mathcal{C}}_g}(h_g) \right )
    \label{eq:c3.1}
\end{align}
Then, combining lemma~\ref{lem:gener}, absolute homogeneity of Rademacher complexity, and the definition of the discrepancy in Eq.~\eqref{eq:dismatch}, we will have
\begin{align}
     &  \sum_{g=1}^{G} \frac{N^i_g}{N_i}  \mathcal{L}_{\mathcal{D}^i_g}(\widehat{h}_g)- \min _{h \in H} \mathcal{L}_{\mathcal{D}_i}(h)  \leq 2\sqrt{\frac{\log \frac{1}{\delta}}{N_i}}+ 4 \sum_{g=1}^G \frac{N^i_g}{N_i} \Re_{\mathcal{C}_g, N_g}(\mathcal{H}) +  2 \sum_{g=1}^{G} \frac{N^i_g}{N_i}  \operatorname{disc}(\mathcal{D}_g^i, \mathcal{C}_{g}) \notag
\end{align}
\end{proof}

\section{Additional Analysis}\label{sec:add_exp}

\subsection{Feature T-SNE map of pre-train model}
In this section, we examine features outputted from various layers of a pre-trained ViT model. As illustrated in Fig.~\ref{fig:early}, features from different classes processed by the early layers of a pre-trained ViT are uniformly distributed on the manifold, indicating shared information across classes. In later layers, the features become more specialized and cluster together, thereby introducing higher heterogeneity in FL. As a result, it validates our motivation in introducing shared prompts into lower layers for common information and group prompts into higher layers for specialized information (Section~\ref{sec:arc}).
\begin{figure}[t!]
    \centering
\includegraphics[width=0.9\linewidth]{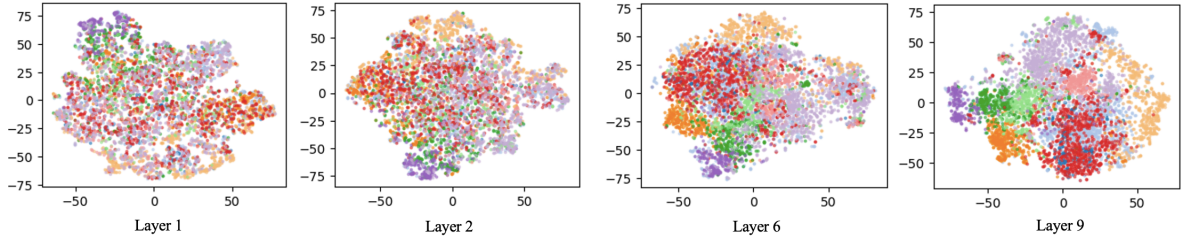}
    \caption{T-SNE maps of CIFAR-100 data features processed by the different layers of the ImageNet-21K pre-trained ViT-16/B model. Data from different coarse classes are labeled with different colors.}\label{fig:early}
    \vspace{-3mm}
\end{figure}
\subsection{The influence of Momentum Ratio}
In this section, we perform an ablation study on the two momentum ratios in Eq.~\ref{eq:sta_key}. We use the CIFAR-100 dataset with $s=10$ as a case study.

\noindent\textbf{Key Momentum Ratio.} We conduct an analysis of the key momentum ratio and begin by applying centralized K-Means to the training data to generate cluster labels for the data. Subsequently, we assess the congruence between the groups learned through our distributed approach and the clusters identified by K-Means. To be specific, we first calculate the normalized contingency matrix~\cite{tsumoto2009contingency} between cluster results from centralized K-Means and \ours{} and obtain the overlapping ratio   $Acc_{overlap}$ by summing the maximum values of each row in this matrix. Then we evaluate the quality of the centralized K-Means result, $Q_{kmeans}$, by calculating the ratio of data from the same class clustered into the same group. Finally, we calculate $\frac{Acc_{overlap}}{Q_{kmeans}}$ to obtain the congruence score. Fig.~\ref{fig:key} demonstrates that \select can match the performance of centralized K-Means. Notably, with a momentum setting of $\alpha_k=0.5$, it achieved its highest congruence score at $86.8\%$.

\noindent\textbf{Group Momentum Ratio.} Based on the optimal key momentum ratio, we study the influence of group momentum ratio $\alpha_g$ on both the global accuracy and the worst local accuracy. Fig.~\ref{fig:groupm} illustrates that an increase in $\alpha_g$ enhances the worst local accuracy, as a higher $\alpha_g$ incorporates more knowledge from previous rounds. As to the global accuracy, initially, it improves attributing to enhanced stability, but it starts to decline when
$\alpha_g$ exceeds 0.5, indicating an over-rely on information from previous rounds. Therefore, the optimal group momentum ratio is $\alpha_k=0.5$. We also report the best baseline performances with two vertical dashed lines, \ours{} can outperform baselines across all momentum ratios.

\noindent Consequently, without further declarations, we set the momentum ratio at 0.5 for both the key $\alpha_k$ and group prompt $\alpha_g$.
\begin{figure}[t!]
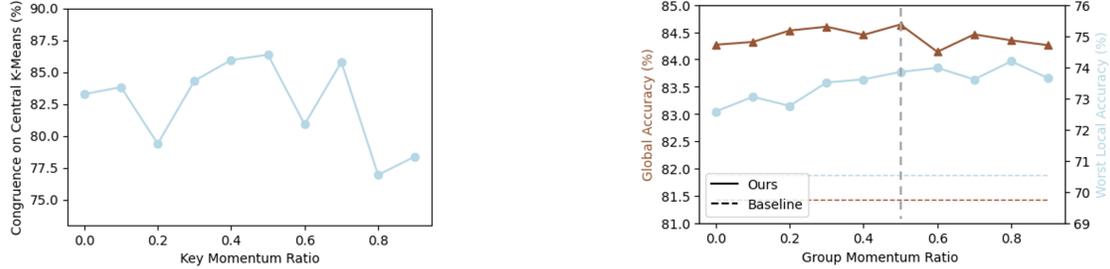

    \centering
    \begin{subfigure}{.45\textwidth}
  \centering
  \includegraphics[width=0.75\linewidth]{image/keymm.pdf}
  \caption{Influence of key momentum ratio $\alpha_k$: We assess the congruence between the groups distributively learned by \select and those identified by centralized K-Means. Optimal congruence is observed at a key momentum ratio of $\alpha_k=0.5$.} 
  \label{fig:key}
\end{subfigure}
 \qquad
\begin{subfigure}{.45\textwidth}
  \centering
  \includegraphics[width=0.8\linewidth]{image/groupmm.pdf}
  \caption{Influence of group momentum ratio $\alpha_g$: We study the influence of the group momentum ratio $\alpha_g$ on global accuracy and worst local accuracy. We highlight the sweet spot using a horizontal dashed line and the best baseline performance with two vertical dashed lines.}
  \label{fig:groupm}
\end{subfigure}%
\caption{Ablation study on key momentum ratio $\alpha_k$ and group momentum ratio $\alpha_g$.}
\vspace{-4mm}
\end{figure}

\section{Description of Heterogeneity}
In this section, we describe the details of label heterogeneity and feature heterogeneity settings and provide examples of them. 
\subsection{Label Heterogeneity}
For CIFAR-100, we follow~\cite{oh2021fedbabu,li2022fedtp} to apply the “Pathological Partition”, where each client is randomly assigned $s$ classes.
The sample rate on client i of selected class $s$ is obtained by $a_{i,c}/\sum_ja_{j,c}$, where $a_{i,c} \sim U(.4,.6)$.  Considering that $s$ equals 10, we illustrate the data distribution for four clients out of a total of 100 in Fig.~\ref{fig:path}. As depicted, $a_{i,c}$ affects the data size of a class and introduces class imbalance within a client. Simultaneously, each client possesses 10 classes, the combination varies across clients, thereby introducing label heterogeneity. As $s$ decreases, the variety of classes available to each client becomes limited, resulting in a restricted label distribution for each client and an increase in the number of samples per class.
\begin{figure}[t!]
    \centering
\includegraphics[width=0.85\linewidth]{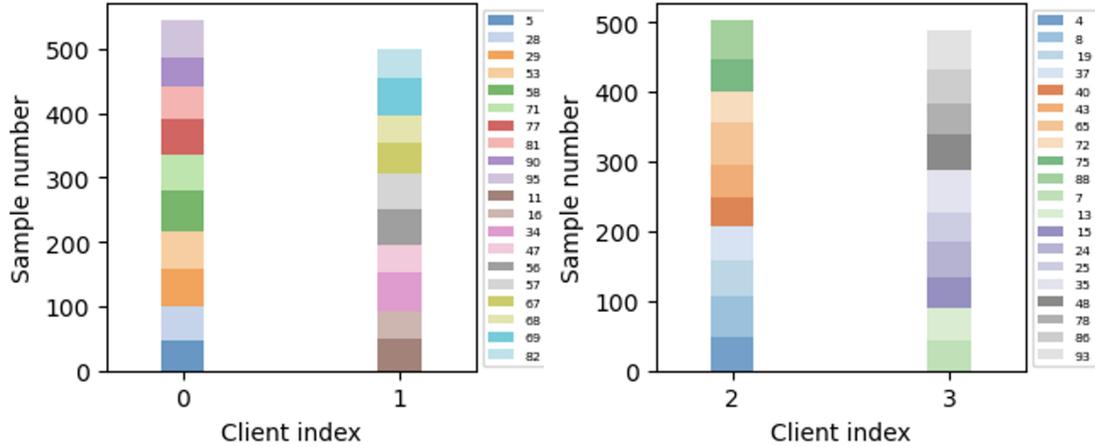}
    \caption{Examples of data distribution with $s=10$ for four clients are presented. Due to the large number of classes (100 classes in CIFAR-100), every two clients are plotted on the same figure, with different colors indicating different classes.}\label{fig:path}
    \vspace{-3mm}
\end{figure}

\noindent Regarding the Five Datasets, we allocate the data and ensure each client receives data from one dataset. Figure~\ref{fig:five} presents image samples from five clients with each client representing one of the five datasets. This approach introduces label heterogeneity due to the unique nature of each dataset.
\subsection{Feature Heterogneity}
For conducting clients with \textit{feature heterogeneity}, we follow the methodologies outlined in the latest benchmark for feature heterogeneity~\cite{yang2023efficient} as well as in the widely-referenced paper~\cite{li2021fedbn}. According to these sources, we assign a data domain to each client, with the total number of clients ($M$) set to 4 for Office-Caltech10 and 6 for DomainNet respectively. Image examples from these datasets are displayed in Fig.\ref{fig:domain} and Fig.\ref{fig:off} for DomainNet and Office-Caltech10, respectively. As illustrated, different clients receive data from the same classes but sourced from various domains, thereby introducing feature heterogeneity.
\begin{figure}[t!]
    \centering
\includegraphics[width=0.6\linewidth]{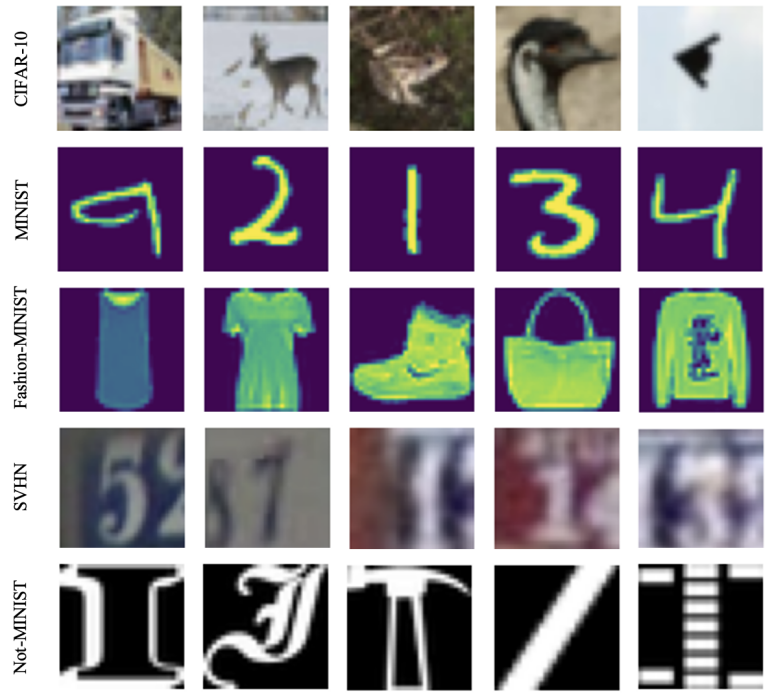}
    \caption{Examples of images from five clients with each client representing one of the five datasets.}\label{fig:five}
\end{figure}

\begin{figure}[t!]
    \centering
\includegraphics[width=0.6\linewidth]{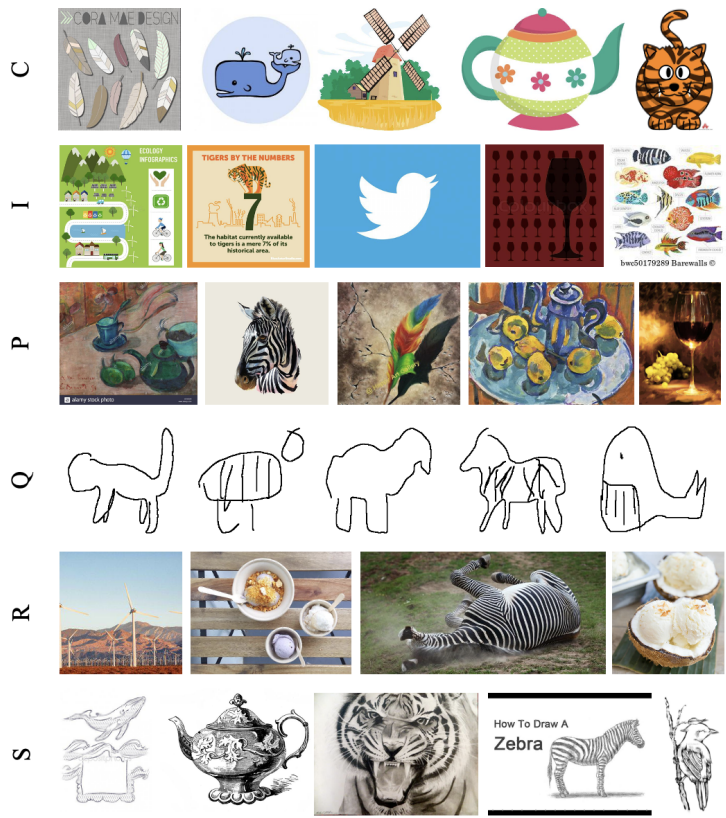}
    \caption{Examples of images from different clients with DomainNet.}\label{fig:domain}
    \vspace{-3mm}
\end{figure}

\begin{figure}[t!]
    \centering
\includegraphics[width=0.6\linewidth]{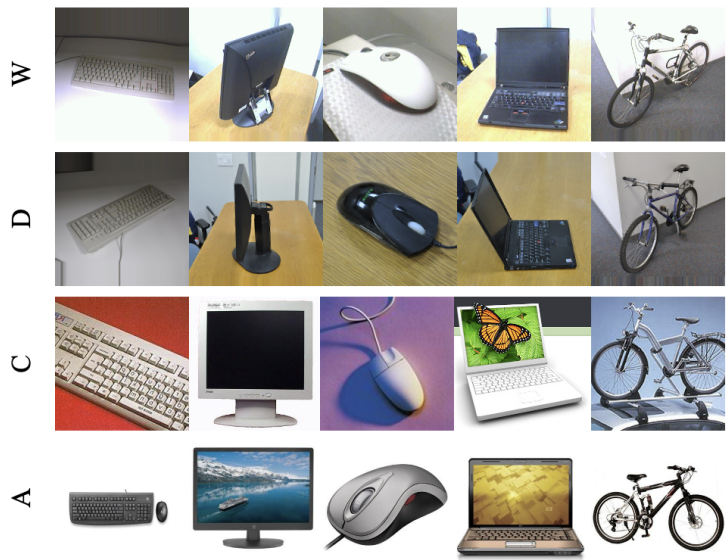}
    \caption{Examples of images from different clients with Office-Caltech10 dataset.}\label{fig:off}
    \vspace{-3mm}
\end{figure}

\end{document}